\documentclass[11pt]{article}
\usepackage[margin=1in]{geometry}
\usepackage[utf8]{inputenc}
\usepackage{times}
\usepackage{graphicx} 
\usepackage{hyperref}
\usepackage[utf8]{inputenc} %
\usepackage{url}            %
\usepackage{booktabs}       %
\usepackage{amsfonts}       %
\usepackage{nicefrac}       %
\usepackage{microtype}      %
\usepackage{xcolor}         %

\usepackage{amsthm,amsmath}
\usepackage{cleveref}
\usepackage{xspace}
\usepackage{color}
\usepackage{bm}
\usepackage{paralist}

\usepackage[ruled,vlined,linesnumbered]{algorithm2e}
\usepackage{listings}
\usepackage[shortlabels]{enumitem}
\usepackage{csquotes}
\usepackage{multirow}

\usepackage{mathtools}
\usepackage{thmtools}
\usepackage{thm-restate}

\usepackage{subcaption}

\graphicspath{{./figures/}}%

\newtheorem{theorem}{Theorem}
\newtheorem{lemma}[theorem]{Lemma}

\theoremstyle{definition}

\DeclareMathOperator{\EX}{\mathbb{E}}%

\newcommand{\IR}{\ensuremath{\mathbb{R}}\xspace}
\newcommand{\opt}{\ensuremath{\mathrm{OPT}}\xspace}

\newcommand{\eps}{\ensuremath{\varepsilon}\xspace}
\newcommand{\ud}{\ensuremath{D}\xspace}
\newcommand{\hs}{\ensuremath{\mathcal{H}}\xspace}
\newcommand{\ths}{\ensuremath{\mathcal{T}}\xspace}

\newcommand{\alg}{\ensuremath{\mathcal{A}}\xspace}

\newcommand{\rob}{\ensuremath{\mathcal{R}}\xspace}
\newcommand{\mrob}{\ensuremath{\overline{\rob}}\xspace}
\newcommand{\pat}{\ensuremath{P}\xspace}
\newcommand{\I}{\ensuremath{\mathcal{I}}\xspace}

\newcommand{\dfsalg}{\ensuremath{\textsc{DFS}}\xspace}
\newcommand{\hdfs}{\ensuremath{\textsc{hDFS}}\xspace}
\newcommand{\fp}{\ensuremath{\textsc{FP}}\xspace}
\newcommand{\nn}{\ensuremath{\textsc{NN}}\xspace}
\newcommand{\block}{\ensuremath{\textsc{Blocking}}\xspace}

\newcommand{\rfp}[1][\lambda]{\ensuremath{\overline{\rob}(\fp,~#1)}\xspace}

\newcommand{\rhdfs}[1][\lambda]{\ensuremath{\overline{\rob}(\hdfs,~#1)}\xspace}

\newcommand{\rblock}[1][\lambda]{\ensuremath{\overline{\rob}(\block,~#1)}\xspace}

\newcommand{\cn}{\ensuremath{\kappa^{N}}\xspace}
\newcommand{\ca}{\ensuremath{\kappa^{\alg}}\xspace}

\newcommand{\nc}{\ensuremath{b_N}\xspace}
\newcommand{\ac}{\ensuremath{b_\alg}\xspace}

\usepackage[textsize=scriptsize]{todonotes}

\newcommand{\nnew}[1]{\textcolor{black}{#1}} %
\newcommand{\anew}[1]{\textcolor{cyan}{#1}}
\newcommand{\lnew}[1]{\textcolor{black}{#1}} %
\newcommand{\fnew}[1]{\textcolor{magenta}{#1}}
\newcommand{\jnew}[1]{\textcolor{orange}{#1}}

\renewcommand{\nnew}[1]{#1} 
\renewcommand{\anew}[1]{#1}
\renewcommand{\lnew}[1]{#1} 
\renewcommand{\fnew}[1]{#1}
\renewcommand{\jnew}[1]{#1}

\title{Robustification of Online Graph Exploration Methods}
\date{}

\author{%
	Franziska Eberle\thanks{Department of Mathematics, London School of Economics, \texttt{f.eberle@lse.ac.uk}} \and Alexander Lindermayr\thanks{Faculty of Mathematics and Computer Science, University of Bremen, \newline \texttt{ \{linderal,nmegow,noelke,jschloet\}@uni-bremen.de}} \and Nicole Megow\footnotemark[2] \and Lukas Nölke\footnotemark[2] \and Jens Schl{\"o}ter\footnotemark[2]
}

\begin{document}
	
	\maketitle
	
	\begin{abstract}
		Exploring unknown environments is a fundamental task in many domains, e.g., robot navigation, network security, and internet search. We initiate the study of a learning-augmented variant of the classical, notoriously hard online graph exploration problem by adding access to machine-learned predictions. We propose an algorithm that naturally integrates predictions into the well-known Nearest Neighbor (\nn) algorithm and significantly outperforms any known online algorithm if the prediction is of high accuracy while maintaining good guarantees when the prediction is of poor quality. We provide theoretical worst-case bounds that gracefully degrade
		with the prediction error, and we complement them by computational  experiments that confirm our results. Further, we extend our concept to a general framework to \emph{robustify} algorithms. By interpolating carefully between a given algorithm and \nn, we prove new performance bounds that leverage the individual good performance on particular inputs while establishing robustness~to~arbitrary~inputs.
	\end{abstract}
	\section{Introduction}\label{sec_intro}

	In online mapping problems, a searcher is tasked to explore an unknown environment and create a complete map of its topology. 
	However, the searcher has only access to local information, e.g., via optical sensors, and must move through the environment to obtain new data.
	Such problems emerge in countless real-life scenarios with a prominent example being the navigation of mobile robots, be it a search-and-rescue robot, an autonomous vacuum cleaner, or a scientific exploration robot in the deep sea or on Mars.
	Less obvious but equally important applications include crawling the Internet or social networks for information and maintaining security of large networks~\cite{Berman96, RaoIJW86, GasieniecR08}.
	
	We investigate the online graph exploration problem on an undirected connected graph~$G=(V,E)$ with~$n$ vertices. %
	Every edge~$e \in E$ has a non-negative cost~$c(e)$, and every vertex~$v \in V$ a unique label. %
	Starting in a designated vertex~$s \in V$, the task of the searcher is to find a tour that visits all vertices of~$G$ and returns to~$s$. 
	A priori, the searcher does not know the graph. Instead, she gains local information when traversing it:
	When for the first time visiting (\emph{exploring}) a vertex, all incident edges, as well as their costs, and the labels of their end points are revealed. 
	This exploration model is also known as \emph{fixed graph scenario}~\cite{KalyanasundaramP93}. 
	In order to explore a vertex~$v$, the searcher may traverse a path consisting of known edges that starts in the searcher's current location and ends in~$v$. 
	When the searcher traverses an edge~$e$, she pays its cost~$c(e)$. 
	The goal is to minimize the total cost. 

	Due to the lack of information, the searcher cannot expect to find an optimal tour. We resort to standard \emph{competitive analysis} to measure the quality of our search algorithms. That is, we compare the length of the tour found by the searcher with the optimal tour that can be found if the graph is known in advance. If the ratio between the costs of these two tours is bounded by~$\rho \geq 1$ for every instance, then we say that the online algorithm is~$\rho$-\emph{competitive}. \nnew{The {\em competitive ratio} of an algorithm is the minimum $\rho$ for which it is~$\rho$-competitive.}
	The \emph{offline} problem of finding an optimal tour on a known graph is the well-known Traveling Salesperson Problem (TSP), which is %
	NP-hard~\cite{LawlerLRS1985}.

	Indeed, it appears extremely difficult to obtain solutions of cost within a constant factor of an optimal tour.
	The best known competitive ratio for arbitrary graphs is~$\mathcal{O}(\log n)$, attained by the following two algorithms.
	The Nearest Neighbor algorithm~(\nn)~\cite{RosenkrantzSL13} greedily explores the unknown vertex closest to the current position. 
	While its performance is usually good in practice~\cite{ApplegateBCC2006-book}, a matching lower bound of~$\Omega(\log n)$ holds even for very simple graphs, e.g., unweighted graphs~\cite{HurkensW04} or trees~\cite{Fritsch21}.
	The second algorithm is the  hierarchical Depth First Search algorithm (\hdfs)~\cite{MegowMS12} that, roughly speaking, executes depth-first searches~(\dfsalg) on subgraphs with low edge costs, thereby limiting its traversal to a minimum spanning tree~(MST). Here, a matching lower bound is attained on a weighted~path.
	
	Only for rather special graph classes it is known how to obtain constant-competitive tours. 
	Notable examples are planar graphs, with a competitive ratio of 16~\cite{KalyanasundaramP93,MegowMS12}, graphs of bounded genus~$g$ with a ratio of~$16(1+2g)$ %
	and graphs with~$k$ distinct weights with a ratio of~$2k$~\cite{MegowMS12}.
	The latter emerges as somewhat of an exception since the \hdfs algorithm achieves a performance that is both, good on a specific yet interesting graph class and still acceptable on arbitrary instances.
	Beyond the above, results are limited to the most basic kind of graphs, such as  unweighted graphs~\cite{MiyazakiMO2009}, cycles and tadpole graphs~\cite{MiyazakiMO2009,BrandtFMW20}, and cactus and unicyclic graphs~\cite{Fritsch21}. 
	Conversely, the best known lower bound on the competitive ratio of an online algorithm is~$10/3$~\cite{BirxDHK2021}. 
	Despite ongoing efforts, it remains a major open question whether there exists an~$O(1)$-competitive exploration algorithm for general~graphs.

	The assumption of having no prior knowledge about the graph may be overly pessimistic. Given the tremendous success of artificial intelligence, %
	we might have access to predictions about good exploration decisions. Such predictions, e.g., machine-learned ones, are typically imperfect; they usually have a good quality but may be arbitrarily bad.
	
	A new line of research is concerned with the design of online algorithms that have access to predictions of unknown quality~\cite{LykourisV18,PurohitSK18,MedinaV17}.
	Ideally, algorithms have the following properties: $(i)$~good predictions lead to a better performance than the best worst-case bound achievable when not having access to predictions; 
	$(ii)$ the algorithm never performs (asymptotically) worse than the best worst-case algorithm even if the prediction is of poor quality; 
	and $(iii)$~the performance gracefully degrades with decreasing prediction quality.
	More formally, we define a parameter~$\eta \geq 0$, called {\em prediction error}, that measures the quality of a given prediction, where $\eta =0$ refers to the case that the prediction is correct, we also say {\em perfect}.
	We assess an algorithm's performance by the competitive ratio as a function of the prediction error. If an algorithm is~$\rho(\eta)$-competitive for some function~$\rho$, we call it~$\alpha$-\emph{consistent} for~$\alpha = \rho(0)$ and $\beta$-\emph{robust} if~$\rho(\eta) \leq \beta$ for any prediction error~$\eta \geq 0$~\cite{PurohitSK18}.

	For the online graph exploration problem, we consider predictions that suggest a known, but unexplored vertex as next target to a learning-augmented algorithm. 
	In other words, a \emph{prediction} is %
	a function that, given the current state of the exploration, outputs an explorable vertex.
	Predictions may be computed dynamically and use all data collected so far, which is what one would expect in practice.
	This rather abstract requirement allows the implementation of 
	various %
	prediction models. 
	In this paper, we consider two kinds of predictions, namely \emph{tour predictions} and \emph{tree predictions}, where the suggested vertex is the next unexplored vertex of a TSP tour or of a \fnew{Depth First Search} (DFS) tour corresponding to some predicted \fnew{spanning} tree, respectively.
	The prediction error~$\eta$ is the difference between the total exploration cost of
	following these per-step suggestions blindly and that of following a perfect prediction w.r.t.\ the given prediction model (tour resp.~tree predictions).

	\subsection{Our results} Our contribution is twofold. Firstly, we present a learning-augmented online algorithm for the graph exploration problem that has a constant competitive ratio when the prediction error is small, while being robust to poor-quality predictions. Our algorithm interpolates carefully between the algorithms \nn and Follow the Prediction (\fp), where the latter blindly follows a given prediction. 

	\begin{theorem}\label{theorem_learningaugmented}
		For any~$\lambda > 0$, there is an algorithm for the online graph exploration problem that uses a predicted spanning tree or tour such that the algorithm is $\kappa(3+4\lambda)$-consistent and $\big(1+\frac{1}{2\lambda}\big)(\lceil\log(n)\rceil+1)$-robust, where~$\kappa =1$, for tour predictions, and~$\kappa =2$, for tree predictions. With growing prediction error, the competitive ratio degrades gracefully with linear dependence on $\eta$.
	\end{theorem}
	The parameter~$\lambda$ can %
	steer the algorithm towards one of the %
	underlying algorithms, e.g., towards \nn when $\lambda \rightarrow \infty$. It %
	reflects our trust in the quality of the provided predictions.

	Further, we show that our predictions (tour and tree) are learnable in the sense of PAC learnability~\cite{Valiant84,vapnik1971} under the assumptions that the given graph is complete and its size known. We show a bound on the sample complexity that is polynomial in the number of nodes and give learning algorithms
	with a polynomial running time in the case of tree predictions and an exponential running time for tour predictions.
	The learnability results also %
	approximately bound the expected prediction error $\eta$, which potentially can be taken into account when setting $\lambda$.

	Our second main result is a general framework to \emph{robustify} algorithms. Given an online algorithm~\alg with a certain worst-case performance for particular classes of instances but unknown, possibly unbounded, performance in general, the robustification framework produces an algorithm with the same good performance on special instances while guaranteeing the best-known worst-case performance $\mathcal O(\log n)$ on general instances. As it turns out, the idea of interpolating between two algorithms that we used to design a learning-augmented algorithm can be generalized to interpolating between the actions of an arbitrary algorithm \alg and \nn, again using the parameter~$\lambda$.
	
	\begin{theorem}\label{theorem_robustification}
		For any $\lambda>0$, there is a robustification framework \rob for the online graph exploration problem that, given an online algorithm~\alg and an instance $\I=(G,s)$, produces a solution of cost at most $R_\I = \min\{(3+4\lambda)\cdot \alg_\I, \; \big(1+\frac{1}{2\lambda}\big) (\lceil\log(n)\rceil+1) \cdot\opt_\I\}$,
		where $\opt_\I$ and $\alg_\I$ denote the cost of an optimal solution and of the one obtained by \alg on instance \I, respectively. 
	\end{theorem}
	
	This seems useful in situations where one may suspect that an instance is of a certain type for which there exist good algorithms. One would like to use a tailored algorithm without sacrificing the general upper bound and good average-case performance of \nn in case the suspicion is wrong. Two illustrative examples are as follows.
	\emph{(i) Planar graphs:} 
	Many spatial networks, e.g., urban street networks, can often be assumed to be (almost) planar~\cite{Barthelemy18,Boeing2020}.
	Here, the graph exploration algorithm  \block~\cite{MegowMS12,KalyanasundaramP93} seems the %
	best choice, given its competitive ratio of $16$.
	Yet, on general instances, the competitive ratio may be unbounded and is known to be worse than~$\omega(\log n)$~\cite{MegowMS12}, underlining the need for robustification.
	\emph{(ii) Bounded number of weights:} Here, \hdfs~\cite{MegowMS12} is the logical choice with a competitive ratio proportional to the number of weights and an asymptotically best-known competitive ratio on general instances. %
	Even here, robustification is useful as it provides the good average-case performance of \nn and the slightly better competitive ratio for general instances.%
	
	Interestingly, when considering the surprisingly good average-case performance of \nn in practice, %
	our robustification framework may %
	also be interpreted to be robustifying \nn and not the specifically tailored algorithm. 
	Either algorithm can possibly make up for the other's shortcomings.
	 
	Our robustification scheme is \nnew{conceptually} in line with 
	other works combining algorithms with different performance characteristics~\cite{MahdianNS12,Fiat1991,%
	Blum2000,Azar1993}. \nnew{However, it is nontrivial to implement such concept for online graph exploration with the particular way in which information is revealed. Since the graph is revealed depending on an algorithm's decisions, the key difficulty lies in handling the cost of different algorithms in different metrics. This fact also prohibits the application of previous learning-augmented algorithms, e.g., for metrical task systems, in our setting.}

	We complement our theoretical results by empirically evaluating the performance of our  algorithms on several real-world instances as well as artificially generated instances. The results confirm the power of using predictions and the effectivity of the robustification framework. 

	\subsection{Further related work}
	The recent introduction of learning-augmented online algorithms~\cite{LykourisV18,MedinaV17,PurohitSK18}
	spawned %
	a multitude of exciting %
	works. These provide %
	methods and concepts for a flurry of problems including, e.g., rent-or-buy problems~\cite{PurohitSK18,GollapudiP19}, scheduling/queuing and bin packing~\cite{PurohitSK18,LattanziLMV20,Mitzenmacher20,AngelopoulosDJKR20,BamasMRS20,AzarLT2021,Im0QP21}, caching~\cite{Rohatgi20,LykourisV18,AntoniadisCE0S20}, the secretary problem~\cite{DuttingLLV21}, revenue optimization~\cite{MedinaV17}, and matching~\cite{KumarPSSV19,LavastidaM0X21}. It is a very active research area. We are not aware of 
	learning-augmented algorithms	
	for online graph exploration.

	Several works empirically study the use of machine learning to solve TSP~\cite{Khalil17,Vinyals15,Bello2017,kool2018} without showing theoretical guarantees. 
	For example, Khalil et al.~\cite{Khalil17} use a combination of reinforcement learning, deep learning, and graph embedding to learn a greedy policy for TSP.
	As the policy might depend on information that is not accessible online,  
	e.g., the degree of an unexplored vertex, and constructs the tour in an offline manner, the results do not directly transfer to our \lnew{online} setting.
	However, similar approaches are conceivable %
	and might be an 
	application for the robustification framework, \jnew{especially since there already exist empirical results in related settings for the exploration of unknown environments.
	For example, one approach~\cite{luperto2019} uses constructive machine learning tools to predict unknown indoor environments, and another approach~\cite{Chiotellis2020} considers online graph exploration as a reinforcement learning problem and solves it using graph neural networks (cf.~\cite{zhou2020}). Since those approaches do not give theoretical guarantees, they are potential applications for the robustification framework.}
	Dai et al.~\cite{Dai2019} consider reinforcement learning for a related problem, where the goal is to explore a maximum number of states in a (possibly unknown) environment using a limited budget.
	In contrast to these algorithmic results, Elmiger et al.~\cite{Elmiger2020} use reinforcement learning to find instances that yield a high competitive ratio for \nn.

	A recent line of research considers data-driven algorithm design~\cite{Balcan2018,Balcan18-1,Balcan18-2,Balcan19,Bhaskara20,Chawla20,Gupta17}. 
	Usually, the task is to select the algorithm with the best expected performance for an unknown distribution over instances from a fixed set of algorithms; see~\cite{Balcan2021} for a survey of recent results. 
	Lavastida et al.~\cite{LavastidaM0X21} combine learning-augmentation with data-driven algorithm design.
	We are not aware of data-driven methods for online graph exploration.
	While our results with regards to PAC learnability of 
	predictions have a similar flavor as the data-driven algorithm design %
	above, there are some major differences. 
	In contrast to data-driven algorithms, 
	we learn predictions with the goal of minimizing the error~$\eta$.
	This error is related to the worst-case guarantees of our algorithms but it does not directly transfer to their expected objective function values.
	Instead, a function depending on the error~(cf.~\Cref{theorem_learningaugmented}) only upper bounds the expected objective values.
	While this may be seen as a disadvantage, it also means that our learned predictions are independent of the used algorithm.

	Another line of research studies the graph exploration problem {\em with advice}~\cite{UngerFB18,DobrevKM2012,KommKKS15}. 
	\nnew{In this model, an algorithm is also equipped with advice that can convey arbitrary information and the goal is to find a competitive solution while using advice of small encoding size. Here, the advice is assumed to be correct which is crucially different from our model.}

	\section{A general robustification scheme}
	\label{sec_robustification}

	In this section, we introduce the robustification scheme~$\rob$ from \Cref{theorem_robustification} that, given an %
	algorithm~$\alg$ for the online graph exploration problem, robustifies its worst-case performance guarantee.
	In the course of the exploration, the set of vertices known to the searcher can be partitioned into \emph{explored} and \emph{unexplored} vertices, i.e., vertices that have already been visited by the searcher, or not, respectively. 
	The robustification scheme uses the algorithm~$\alg$ as a blackbox. %
	That is, we treat~$\alg$ as a function that, given the current position, %
	currently known subgraph, and %
	set of already explored vertices, returns the next vertex to explore.
	The learning-augmented algorithm from \Cref{theorem_learningaugmented} emerges as an application of the robustification scheme and is discussed in \Cref{sec_learningaugmented}.
	
	\subsection{The robustification scheme}
	
	Intuitively, the robustification scheme \rob, summarized in Algorithm~\ref{Alg_robustification}, balances the execution of algorithm~$\alg$ with that of \nn by executing the algorithms in alternating phases. 
	These phases are budgeted so that their costs are roughly proportional to each other, with a parameter $\lambda>0$ dictating the proportionality.
	Specifically, whenever~$\alg$ is at position~$v$ and about to explore a vertex~$u$ via some path~$\pat_u^\alg$, we interrupt \alg and, instead, start from $v$ a phase of exploration via \nn.
	This phase ends when the cost incurred by \nn %
	reaches~$\lambda c(\pat_u^\alg)$ or when \nn is about to explore~$u$ (Lines~\ref{line_robust_while} to~\ref{line_robust_nn_end}).
	Only afterwards does the scheme explore the vertex~$u$ and resumes exploration via \alg (Line~\ref{line_robust_A_traversal}). 
	
	\begin{algorithm}
		\DontPrintSemicolon
		\caption{Robustification scheme~$\rob$.}
		\label{Alg_robustification}
		\KwIn{Partially explored graph~$G$, start vertex~$s$, algorithm~$\alg$, and parameter~$\lambda > 0$\;}
		$G_\alg \gets G$\tcp*[r]{subgraph revealed to~$\alg$}
		\While{$G$ has an unexplored vertex}{
			$u \gets$ next unexplored node to be visited by \alg, computed via Algorithm~\ref{ALG_A}\label{line_robust_simulate}\;
			$\pat^\alg_u \gets$ shortest~$s$-$u$-path in~$G$\label{line_robust_A_path}\;
			$u' \gets$ nearest unexplored neighbor of~$s$, $\;b \gets 0$\label{line_robust_nn_start}\;
			\While{$b < \lambda \cdot c(\pat_u^\alg)$ and~$s \not= u$}{\label{line_robust_while}
				traverse a
				shortest~$s$-$u'$-path $\pat_{u'}$ and update $G$\label{line_robust_nn_traversal}\;	
				$s \gets u'$, $\;b \gets b + c(\pat_{u'})$\;
				$u' \gets$ nearest unexplored neighbor of~$s$\label{line_robust_nn_end}\;
			} 
			traverse a shortest known path to~$u$, set~$s \gets u$ and update $G$\label{line_robust_A_traversal}\;
			update $G_\alg$ to reflect exploration of $u$\; 
			
		}
		traverse a shortest path in $G$ to the start vertex\;
	\end{algorithm}

	Note that we do not reveal to~$\alg$ information gained by exploring vertices during the nearest-neighbor phase (Lines~\ref{line_robust_while} to~\ref{line_robust_nn_end}).
	If~$\alg$ decides to explore a vertex~$u$ next that is already known to~$\rob$, %
	we only simulate~$\alg$ without actually executing any traversals (Line~\ref{line_robust_simulate} resp.\ Algorithm~\ref{ALG_A}).
	This is possible since the new information that~$\alg$ would obtain by exploring~$u$ is already known to~$\rob$.

	\begin{algorithm}
		\DontPrintSemicolon
		\caption{Computes next unexplored node visited by~$\alg$ and updates $G_\alg$.}
		\label{ALG_A}
		\KwIn{Partially explored graph~$G$, blackbox graph~$G_\alg$, start vertex~$s$ and algorithm~$\alg$\;}
		\While{next vertex~$u$ explored by~$\alg$, given~$G_\alg$ and~$s$, is explored in~$G$}{
			update~$G_\alg$ by adding previously unknown edges incident to~$u$, and mark~$u$ as explored\;
			$s \gets u$\;
		}
		\Return $u$, $G_\alg$\;
	\end{algorithm}

	Recall that, given an online algorithm \alg and a graph exploration instance \I, the terms $\alg_\I$ and $\opt_\I$ refer to the costs incurred by \alg and an optimal solution on instance \I, respectively.
	To prove Theorem~\ref{theorem_robustification}, we bound the cost incurred by \rob during the \nn %
	phases in terms of~$\opt_\I$. 
	
	\begin{restatable}{lemma}{rosenkrantz}
		\label{lemma_nearest_neighbor_cost}
		The cost $\cn$ of all traversals in Line~\ref{line_robust_nn_traversal} is at most $\frac{1}{2} (\lceil\log n\rceil+1) \opt_\I$.
	\end{restatable}
	The lemma can be shown by following the approach of~\cite{RosenkrantzSL13}.
	While there one consecutive nearest-neighbor search is considered, our algorithm starts and executes multiple (incomplete) nearest-neighbor searches with different starting points.
	In the adapted proof of the lemma, we use the following auxiliary result.
	\begin{lemma}[Lemma~1 in~\cite{RosenkrantzSL13}]\label{lem:Rosenkrantz}
		Let~$l \in \IR^V$ be such that for an instance $\I$ of the graph exploration problem the following properties hold:
		\begin{enumerate}[(i),noitemsep]
			\item $c(P_{\{v,u\}}) \geq \min\{l_v,l_u\}$, for all~$v,u \in V$, where~$P_{\{v,u\}}$ is a shortest~$u$-$v$-path,  and
			\item $l_v \leq \frac 1 2 \opt_\I$, for all~$v \in V$. 
		\end{enumerate}
		Then,~$\sum_v l_v \leq \frac 1 2 (\lceil\log n\rceil+1) \cdot \opt_\I$.
	\end{lemma}
	\begin{proof}[Proof of \Cref{lemma_nearest_neighbor_cost}]
		We assign to each~$v \in V$ a value~$l_v$ such that~$\cn_i = \sum_{v\in V} l_v$ and both conditions of Lemma~\ref{lem:Rosenkrantz} hold, which then implies the statement of \Cref{lemma_nearest_neighbor_cost}.
		
		Consider a current vertex~$s$ before a traversal in the \nn phase, i.e., a traversal in Line~\ref{line_robust_nn_traversal}.
		Observe that in this manner no vertex is considered twice since~$s$ is always the most recently newly explored vertex in~$G$ (or, in case of the very first iteration, the initially given start vertex). 	
		Let $l_s = c(P)$, where $P$ is the shortest~$v$-$u'$-path traversed by \nn in Line~\ref{line_robust_nn_traversal} and $l_v = 0$ for all remaining vertices. Then, $\cn_i = \sum_{v\in V} l_v$.
		It remains to show that both conditions of \Cref{lem:Rosenkrantz} are satisfied.

		Condition~(i): Suppose~$l_v, l_u > 0$ as otherwise the condition trivially holds. Moreover, assume without loss of generality that~$v$ was explored before~$u$. Thus, when traversing to the nearest neighbor~$u'$ of~$v$, the vertex~$u$ was still unexplored. 
		By definition,~$l_v$ corresponds to the cost of a shortest~$v$-$u'$-path.
		Since~$u$ was still unexplored but~$u'$ was selected as the nearest neighbor of~$v$, it follows that the shortest path between~$v$ and~$u$ cannot be shorter than the shortest path between~$v$ and~$u'$. Thus, the shortest path between~$v$ and~$u$ has cost at least~$l_v$.
		
		Condition~(ii): Again, assume~$l_v > 0$. 
		The value~$l_v$ corresponds to the cost of the shortest path between~$v$ and some vertex~$u$.
		The condition easily follows when viewing~$\opt_\I$ as two paths connecting~$v$ and~$u$.	
	\end{proof}
	
	Using \Cref{lemma_nearest_neighbor_cost}, we show the theorem.
	
	\begin{proof}[Proof of Theorem~\ref{theorem_robustification}]
		Fix $\lambda > 0$, an algorithm \alg for the graph exploration problem, and an instance~$\I$.
		Denote by $\rob_\I$ the cost incurred on instance \I by the robustification scheme \rob applied to \alg with parameter $\lambda$.
		We show~$\rob_\I \le (3+4\lambda) \alg_\I$ and~$\rob_\I \le \big(1+\frac{1}{2\lambda}\big) (\lceil\log(n)\rceil+1) \opt_\I$ separately.
		For each iteration~$i$ of the outer while loop, denote by~$\cn_i$ the traversal cost incurred by the inner while loop (Line~\ref{line_robust_nn_traversal}), and by~$\ca_i$ the cost of the traversal in Line~\ref{line_robust_A_traversal}.
		Then,~$\rob_\I = \sum_{i} (\ca_i + \cn_i)$.

		\emph{Proof of~${\rob_I \le (3+4\lambda) \alg_I}$:~~}
		For iteration~$i$ of the outer while loop, in which~$\alg$ wants to explore~$u$, let~$\pat_i^\alg$ be the shortest~$s$-$u$-path in Line~\ref{line_robust_A_path}.
		Since Line~\ref{line_robust_simulate} resp.\ Algorithm~\ref{ALG_A} only simulate traversals, the~$\pat_i^\alg$ may not match the actual traversals that are due to algorithm~$\alg$.
		Specifically, it might be the case that~\mbox{$\sum_{i} c(\pat_i^\alg) \neq \alg_\I$.}
		However, since~$\pat_i^\alg$ is a shortest path in the currently known graph~$G$ which contains~$G_\alg$ after executing the simulated traversals, it cannot exceed the sum of the corresponding simulated traversals. Thus,~\mbox{$\sum_{i} c(\pat_i^\alg) \le \alg_\I$}.

		Consider an %
		iteration~$i$ and the traversal cost~$\cn_i + \ca_i$ incurred during this iteration. %
		We start by upper bounding~$\cn_i$.
		Let~$\kappa_i'$ be the cost of the inner while loop (Lines~\ref{line_robust_while} to~\ref{line_robust_nn_end}) \emph{excluding} the last iteration.
		By definition,~$\kappa_i' < \lambda \cdot c(\pat_i^\alg)$.
		
		Let~$\pat_i$ be the path traversed in the last iteration of the inner while loop,~$s'$ its start vertex, and~$u'$ its end vertex. 
		Recall that~$u$ is the endpoint of~$\pat_i^\alg$. 
		Before executing the inner while loop, %
		the cost of the shortest path from the current vertex to~$u$ was~$c(\pat_i^\alg)$.
		By executing the inner while loop, excluding the last iteration, the cost of the shortest path from the new current vertex to~$u$ can only increase by at most~$\kappa_i' < \lambda \cdot c(\pat_i^\alg)$ compared to the cost of~$\pat_i^\alg$.
		Since~$\pat_i$ is the path to the nearest neighbor of the current vertex,~$c(\pat_i)$ cannot be larger than the cost of the shortest path to vertex~$u$. 
		Thus,~$c(\pat_i) \le c(\pat_i^\alg) + \kappa_i' < (1+\lambda) \cdot c(\pat_i^\alg)$, and~$\cn_i = \kappa_i' + c(\pat_i) \le (1+2\lambda) \cdot c(\pat_i^\alg)$.
		
		To bound $\ca_i$, consider the traversal of the shortest~$s$-$u$-path in Line~\ref{line_robust_A_traversal}. 
		Before executing the inner while loop, the cost of the shortest path from the current vertex to~$u$ was~$c(\pat_i^\alg)$.
		By executing the while loop, the cost of the shortest path from the new current vertex to~$u$ can increase by at most~$\cn_i \le (1+2\lambda) \cdot c(\pat_i^\alg)$ compared to~$c(\pat_i^\alg)$.
		This implies~$\ca_i \le (2+2\lambda) \cdot c(\pat_i^\alg)$.
		Using~$\ca_i + \cn_i \le (3+4\lambda) \cdot c(\pat_i^\alg)$, we conclude
		\[
			\rob_\I = \sum_{i} \big(\ca_i + \cn_i\big) \le (3+4\lambda) \sum_{i} c(\pat_i^\alg) \le (3+4\lambda) \alg_\I.
		\]

		\emph{Proof  of~$\, \rob_\I \le \big(1+\frac{1}{2\lambda}\big) (\lceil\log(n)\rceil+1) \opt_\I $:~~}
		We have~$\ca_i \le c(\pat_i^\alg) + \cn_i$. If the inner while loop was aborted due to $s=u$, then $\ca_i = 0$. Otherwise, $\cn_i \ge \lambda \cdot c(\pat_i^\alg)$, and thus, $\ca_i \le (1+\frac{1}{\lambda}) \cn_i$. We conclude $\rob_\I = \cn+\sum_i\ca_i \leq \big(2+\frac{1}{\lambda}\big) \cn$. Lemma~\ref{lemma_nearest_neighbor_cost} directly implies the result. 
	\end{proof}
	
	\subsection{Reducing the overhead for switching %
	algorithms}
	\label{subsec:modifiedR}

	The robustification scheme~$\rob$ balances between the execution of a blackbox algorithm~$\alg$ and a nearest-neighbor search with the parameter~$\lambda$, that allows us to configure the proportion at which the algorithms are executed. Even for arbitrarily small~$\lambda > 0$, the worst-case cost of~$\rob$ on instance~$\I$ is still~$(3+4\lambda) \alg_\I \approx 3 \alg_\I$. The loss of the factor~$3$ is due to the overhead created by switching between the execution of algorithm~$\alg$ and the nearest-neighbor search.
	Hence, we modify~$\rob$ to reduce this overhead. For a fixed~$\lambda$, this leads to a slighly worse worst-case guarantee. However, instances leading to a cost of roughly $3 \alg_\I$ are very particular, and here the modification significantly improves the average-case performance as our experimental results in~\Cref{sec_experiments}
	show.
	
	We first describe how the modified robustification scheme~$\overline{\rob}$ (cf. Algorithm~\ref{Alg_m_robustification}) differs from Algorithm~\ref{Alg_robustification}.
	Intuitively, the modified robustification scheme~$\overline{\rob}$ reduces the overhead for switching between algorithms by extending the individual phases.
	Specifically, to not continuously interrupt~\alg, we introduce a budget for its execution that is now proportional to the cost of \emph{all} past nearest neighbor phases. 
	In turn, we 
	also increase the budget for \nn to match these larger phases of executing \alg:
	
	The modified scheme uses the budget variable $\nc$ to keep track of the total cost incurred by following~\nn (cf. Lines~\ref{line_mrobust_nnbudget} and~\ref{line_mrobust_nnbudget_increase}).
	At the beginning of each iteration of the outer while loop, the algorithm exploits this budget to follow $\alg$ until the next traversal would surpass the parametrized budget $\frac{1}{\lambda} \cdot \nc$ (cf. Lines~\ref{line_mrobust_aloop_start} and~\ref{line_mrobust_aloop_end}).
	The cost incurred by those traversals in the current iteration of the outer while loop is stored in variable $\ac$ (cf. Lines~\ref{line_mrobust_ac} and~\ref{line_mrobust_ac_increase}).
	
	In order to compensate for this additional phase of following $\alg$, the while loop of Line~\ref{line_mrobust_while} follows $\nn$ until either the cost is at least $\lambda \cdot (\ac + c(P_u^\alg))$ or the cost is at least $\lambda \cdot \ac$ and the last explored vertex is the previous target vertex of both \nn and \alg.
	In contrast to the original robustification scheme, it can happen that the target vertex $u$ of $\alg$ is explored during the execution of $\nn$ but we still need to continue following $\nn$ in order to reach a cost of at least $\lambda \cdot \ac$.
	To handle those situations, the algorithm recomputes the next target vertex $u$ of $\alg$ and the corresponding path $\pat^\alg_u$ (cf. Lines~\ref{line_mrobust_recompute_u_start} to~\ref{line_mrobust_recompute_u_end}).
	
	Finally, again in contrast to the original version, the modified robustification scheme might explore the final vertex of the graph before the end of the current iteration of the outer while loop.
	Thus, the algorithm ensures to not execute further traversals if the graph is already fully explored, except for the return to the start vertex in Line~\ref{line_mrobust_return}.

	\begin{algorithm}[h]
	\DontPrintSemicolon
	\caption{Modified robustification scheme~$\overline{\rob}$.}
	\label{Alg_m_robustification}
	\KwIn{Partially explored graph~$G$, start vertex~$s$, algorithm~$\alg$, and parameter~$\lambda > 0$\;}
	$G_\alg \gets G$\tcp*[r]{subgraph revealed to~$\alg$}
	$\nc \gets$ $0$\label{line_mrobust_nnbudget}\;
	\While{$G$ has an unexplored vertex}{
		$\ac \gets 0$ \label{line_mrobust_ac}\;
		$u \gets$ next unexplored node to be visited by \alg, computed 
		via Algorithm~\ref{ALG_A}\label{line_mrobust_simulate}\;
		$\pat^\alg_u \gets$ shortest~$s$-$u$-path in~$G$\label{line_mrobust_A_path}\;
		\While{$\ac + c(\pat^{\alg}_u) \le \frac{1}{\lambda} \cdot \nc$ and $G$ has an unexplored vertex}{\label{line_mrobust_aloop_start}
			traverse $\pat^{\alg}_u$ and update $G$ and $G_\alg$\label{line_mrobust_ac_additional}\;	
			$s \gets u$, $\ac \gets \ac + c(P_u^\alg)$\label{line_mrobust_ac_increase}\;
			$u \gets$ next unexplored node to be visited by \alg, computed 
			via Algorithm~\ref{ALG_A}\;
			$\pat^\alg_u \gets$ shortest~$s$-$u$-path in~$G$\label{line_mrobust_aloop_end}\;
		}
		$u' \gets$ nearest unexplored neighbor of~$s$, $\;b \gets 0$\label{line_mrobust_nn_start}\;
		\While{(($b < \lambda \cdot (\ac + c(\pat_u^\alg) )$ and~$s \not= u$) or ($b < \lambda \cdot \ac$ and $s = u $)) and $G$ has an unexplored vertex}{\label{line_mrobust_while}
			traverse a
			shortest~$s$-$u'$-path $\pat_{u'}$ and update $G$\label{line_mrobust_nn_traversal}\;	
			$s \gets u'$, $\;b \gets b + c(\pat_{u'})$\;
			\If{$s = u$ and $b < \lambda \cdot \ac$}{\label{line_mrobust_recompute_u_start}
				update $G_\alg$ to reflect exploration of $u$\;
				$u \gets$ next unexplored node to be visited by \alg, computed 
				via Algorithm~\ref{ALG_A}\;
				$\pat^\alg_u \gets$ shortest~$s$-$u$-path in~$G$\;	\label{line_mrobust_recompute_u_end}
			}
			$u' \gets$ nearest unexplored neighbor of~$s$\label{line_mrobust_nn_end}\;
		} 
		$\nc \gets \nc + b$\label{line_mrobust_nnbudget_increase}\;
		\If{$G$ has an unexplored vertex}{
			traverse a shortest known path to~$u$, set~$s \gets u$ and update $G$\label{line_mrobust_A_traversal}\;
			update $G_\alg$ to reflect exploration of $u$\; 
		}		
	}
	traverse a shortest part in $G$ to the start vertex\label{line_mrobust_return}\;
	\end{algorithm}

	The analysis of the modified robustification scheme~$\overline{\rob}$ remains essentially %
	the same. 
	The second term in the minimum of the worst-case guarantee is slightly higher since the algorithm might terminate directly after executing an extended phase of \alg without giving \nn a chance to compensate for possible errors. 
	The first term, i.e.,~$(3+4\lambda) \alg_\I$, remains the same since the cost of an iteration of the outer while loop remains bounded by $(3+4\lambda)$ times the cost of an extended phase of \alg plus the cost of reaching the next target of \alg after that phase. %

\begin{restatable}{theorem}{modifiedRobustification}
	\label{theorem_modified_robustification}
	Given an algorithm~$\alg$ for the online graph exploration problem and %
	$\lambda > 0$, the modified robustification scheme~$\overline{\rob}$ solves a graph exploration instance~$\I$ with cost
	\[
	\overline{\rob}_\I \le \min\Big\{(3+4\lambda) \alg_\I, \left(1+\frac{1}{\lambda}\right) (\lceil\log(n)\rceil+1) \opt_\I \Big\}.
	\]
\end{restatable}
\begin{proof}
	Fix $\lambda > 0$, an algorithm \alg for the graph exploration problem, and an instance~$\I$.
	Denote by $\mrob_\I$ the cost incurred on instance \I by the modified robustification scheme $\overline{\rob}$ applied to \alg with parameter $\lambda$.
	We show~$\mrob_\I \le (3+4\lambda) \alg_\I$ and~$\mrob_\I \le \frac{1}{2} (1+\frac{1}{\lambda}) (\lceil\log(n)\rceil+1) \opt_\I$ separately.
	
	For each iteration~$i$ of the outer while loop, denote by~$\cn_i$ the traversal cost incurred by the inner while loop (Line~\ref{line_robust_nn_traversal}), and by~$\ca_i$ the cost of the traversal in Lines~\ref{line_mrobust_ac_additional} and~\ref{line_mrobust_A_traversal}.
	Then,~$\rob_\I = \sum_{i} (\ca_i + \cn_i)$.	
	
	\textit{Proof of~${\rob_I \le (3+4\lambda) \alg_I}$:~~}
	For each iteration~$i$ of the outer while loop, let~$\pat_i^\alg$ be the shortest~$s$-$u$-path $P_u^\alg$ that is considered in the last execution of one of the Lines~\ref{line_mrobust_A_path},~\ref{line_mrobust_aloop_end} or~\ref{line_mrobust_recompute_u_end}.
	Further, let $\ac^i$ denote the traversal cost incurred by all executions of Line~\ref{line_mrobust_ac_additional} in iteration $i$.
	
	Since Algorithm~\ref{ALG_A} only simulates traversals and not all traversals of \alg are considered due to the re-computation of $\pat^\alg_u$ (cf.~Lines~\ref{line_mrobust_A_path},~\ref{line_mrobust_aloop_end} and~\ref{line_mrobust_recompute_u_end}), the cost $\ac^i + c(\pat_i^\alg)$ may not match the actual traversal cost that are due to algorithm~$\alg$ in the iteration.
	Specifically, it might be the case that~$\sum_{i} \ac^i + c(\pat_i^\alg) \neq \alg_\I$.
	However, since~$\pat_i^\alg$ is a shortest path in the currently known graph~$G$ which contains~$G_\alg$
	after executing the simulated traversals, it must be shorter than the sum of the corresponding simulated traversals. 	
	Thus,~\mbox{$\sum_{i} \ac^i + c(\pat_i^\alg) \le \alg_\I$.}
	
	Consider an arbitrary iteration~$i$ and the traversal cost incurred during this iteration, i.e.,~$\cn_i + \ca_i$.
	We start by upper bounding~$\cn_i$.
	Let~$\kappa_i'$ be the cost of the inner while loop (Lines~\ref{line_mrobust_while} to~\ref{line_mrobust_nn_end}) \emph{excluding} the last iteration.
	By definition,~$\kappa_i' < \lambda \cdot (\ac^i + c(\pat_i^\alg))$ %
	
	Let~$\pat_i$ be the path traversed in the last iteration of the inner while loop,~$s'$ its start vertex, and~$u'$ its end vertex. 
	Recall that~$u$ is the endpoint of~$\pat_i^\alg$. 
	Before executing the iterations of the inner while loop that remain to be executed after $\pat_i^\alg$ was computed, %
	the cost of the shortest path from the current vertex to~$u$ was~$c(\pat_i^\alg)$.
	By executing the remaining iterations of the inner while loop, excluding the last iteration, the cost of the shortest path from the new current vertex to~$u$ can only increase by at most~$\kappa_i' < \lambda \cdot (\ac^i + c(\pat_i^\alg))$ compared to the cost of~$\pat_i^\alg$.
	Since~$\pat_i$ is the path to the nearest neighbor of the current vertex,~$c(\pat_i)$ cannot be larger than the cost of the shortest path to vertex~$u$. 
	Thus,~$c(\pat_i) \le c(\pat_i^\alg) + \kappa_i' < (1+\lambda) \cdot c(\pat_i^\alg) + \lambda \cdot \ac^i$, and
	\[
		\cn_i = \kappa_i' + c(\pat_i) \le(1+2\lambda) \cdot c(\pat_i^\alg) + 2\lambda \cdot \ac^i.
	\]
	
	Consider $\ca_i = \ac^i + l_i$ where $l_i$ denotes the traversal cost of Line~\ref{line_mrobust_A_traversal}.
	To bound $l_i$, consider the traversal of the shortest~$s$-$u$-path in Line~\ref{line_mrobust_A_traversal}. 
	When $\pat^\alg_i$ was computed in the last execution of the lines Lines~\ref{line_mrobust_A_path},~\ref{line_mrobust_aloop_end} or~\ref{line_mrobust_recompute_u_end}, the cost of the shortest path from the current vertex to~$u$ was~$c(\pat_i^\alg)$.
	Afterwards, the cost of the shortest path from the new current vertex to~$u$ can increase only because of traversals in Line~\ref{line_mrobust_nn_traversal}.
	Thus, the cost can increase by at most~$\cn_i \le(1+2\lambda) \cdot c(\pat_i^\alg) + 2\lambda \cdot \ac^i$ compared to~$c(\pat_i^\alg)$.
	This implies~$l_i \le c(\pat^\alg_i) + \cn_i \le (2+2\lambda) \cdot c(\pat_i^\alg) + 2\lambda \cdot \ac^i$ and, therefore, $\ca_i = \ac^i + l_i \le (2+2\lambda) \cdot c(\pat_i^\alg) + (1+2\lambda) \cdot \ac^i$.
	
	It follows~$\ca_i + \cn_i \le (3+4\lambda) \cdot c(\pat_i^\alg) + (1+4\lambda) \cdot \ac^i$. 
	Note that, if the iteration is aborted early because all vertices are explored, the bound on $\ca_i + \cn_i$ in terms of $c(\pat_i^\alg) + \ac^i$ only improves.
	We conclude
	\begin{align*}
	\mrob_\I &= \sum_{i} \ca_i + \cn_i \\ &\le (3+4\lambda) \Big(\sum_{i} c(\pat_i^\alg)\Big) + (1+4\lambda)  \Big(\sum_i \ac^i\Big)\\ &\le (3+4\lambda) \cdot \alg_\I.
	\end{align*}
	Observe that the worst-case of $\mrob_\I \approx (3+4\lambda) \cdot \alg_\I$ can only occur if the cost $\sum_{i} c(\pat_i^\alg)$ dominates the cost $\sum_i \ac^i$, i.e., if the cost incurred by the last traversals of \alg per iteration dominates the cost of all other traversals by \alg.  
	For larger $\sum_i \ac^i$, the bound on $\mrob_\I$ improves, which is a possible explanation for the success of $\mrob$ in the experimental analysis.
	
	\textit{Proof  of~$ \rob_\I \le \big(1 + \frac1\lambda\big) (\lceil\log(n)\rceil+1) \opt_\I $:} 
	Consider an arbitrary iteration $i$ of the outer while loop that is \emph{not} the last one, i.e., $i$ is a full iteration, and
	consider again $\ca_i = \ac^i + l_i$ where $l_i$ denotes the traversal cost of Line~\ref{line_mrobust_A_traversal}.	
	
	As argued before, we have~$l_i \le c(\pat_i^\alg) + \cn_i$.
	If the inner while loop was aborted because $s=u$ and~$\cn_i \ge \lambda \cdot \ac^i$, then $l_i = 0$.
	This implies $\ca_i = \ac^i + l_i = \ac^i \le \frac{1}{\lambda} \cdot \cn_i < (1 + \frac{1}{\lambda}) \cdot \cn_i$.
	If the inner while loop was aborted because $s\not=u$ and $\cn_i \ge \lambda \cdot (\ac^i + c(\pat^\alg_i) )$, then $(\ac^i + c(\pat^\alg_i)) \le \frac{1}{\lambda} \cdot \cn_i$ which implies~$\ca_i = \ac^i + l_i \le \ac^i + c(\pat^\alg_i) + \cn_i \le (1+\frac{1}{\lambda}) \cdot \cn_i$.
	Thus, we can conclude $\cn_i + \ca_i \le (2 + \frac{1}{\lambda})  \cdot \cn_i$.
	
	Let $\bar{\kappa}^N = \sum_{i} \cn_i$ and $\bar{\kappa}^\alg = \sum_i \ca_i$ \emph{ignoring} the last iteration $j$ of the outer while loop. 
	Then, $\cn = \bar{\kappa}^N + \cn_j$, $\ca = \bar{\kappa}^\alg + \ca_j$ and $\mrob_\I = \bar{\kappa}^N + \cn_j + \bar{\kappa}^\alg + \ca_j \le (2+ \frac{1}{\lambda}) \cdot \bar{\kappa}^N + \ca_j + \cn_j$.
	When bounding $\mrob_\I$ in terms of $\cn$, the worst-case occurs if $\ca_j$ is as large as possible and $\cn_j$ is as small as possible. 
	Thus, the worst-case occurs if the graph is explored before the last iteration executes the second inner while loop.
	In that case, $\cn_j = 0$ and, by definition of the first inner while loop (cf. Line~\ref{line_mrobust_aloop_start}), $\ca_j \le \frac{1}{\lambda} \cdot \bar{\kappa}^N$.
	We can conclude that $\mrob_\I \leq (2+ \frac{1}{\lambda}) \cdot \bar{\kappa}^N + \ca_j + \cn_j \le (2+ \frac{2}{\lambda}) \cdot \cn$. 
	Since the proof of Lemma~\ref{lemma_nearest_neighbor_cost} is not affected by the modifications of $\mrob$, we can apply Lemma~\ref{lemma_nearest_neighbor_cost} on $\cn$. 
	This implies the desired result.	
\end{proof}
	
	\section{Online graph exploration with untrusted predictions}
	\label{sec_learningaugmented}
	In this section, we apply the previously introduced robustification scheme in the context of learning-augmented algorithms and present an algorithm that satisfies the criteria in \Cref{theorem_learningaugmented}.
	This algorithm is provided with untrusted predictions that come in the form of a fixed exploration order of the vertices, given by a spanning tree %
	or a tour. 
	The learnability of such predictions is discussed in \Cref{sec_pac_learnability}.

	\subsection{A learning-augmented online algorithm}
	We consider prediction models which upfront fix an exploration order~$\tau_p$ of the vertices. Such an order can be predicted directly (\emph{tour predictions}) or is given by the traversal order of a \emph{Depth First Search} (\dfsalg) on a predicted spanning tree~$T_p$ (\emph{tree predictions}). 
	Recall that a prediction is a function that outputs for a given exploration state an explorable vertex, and, given an order~$\tau_p$, this vertex is the first unexplored vertex in~$\tau_p$ w.r.t.\ the current state. Due to this mapping, we also call~$\tau_p$ a prediction.
	Denote by~$c(\tau_p)$ the cost of executing \fp with the prediction~$\tau_p$. 
	A \emph{perfect prediction}~$\tau^*$ is in the case of tour predictions an optimal tour, and in the case of tree predictions the \dfsalg traversal order of a \fnew{Minimum Spanning Tree} (MST)~$T^*$. The prediction error is defined as~$\eta = c(\tau_p) - c(\tau^*)$.
	Regarding tree predictions, the definition of \dfsalg ensures that each edge in~$T^*$ is traversed at most twice, thus~$c(\tau^*) \le 2c(T^*)$. Using~$c(T^*) \le \opt_\I$, this implies the following lemma. 
	\begin{lemma}\label{lemma_mst_dfs}
		For 
		an instance~$\I$, following tree predictions has cost~$\fp_\I \le 2 \opt_\I + \eta$.
	\end{lemma}
	Naively calling \fp might lead to an arbitrarily bad competitive ratio of \fp.
	Luckily, \Cref{theorem_robustification} provides us with a tool to mitigate this possibility and cure \fp 
	of its naivety. 
	Using \fp %
	within the robustification scheme of Algorithm~\ref{Alg_robustification} allows us to bound the worst-case performance.
	\lnew{Denote by $\rob(\fp,G)$ the performance of this strategy.} %
	With Lemma~\ref{lemma_mst_dfs}, and given an instance \I with tree predictions, \lnew{we can upper bound $\rob(\fp,G)$ by}
	\[
	\min\{(3+4\lambda)(\kappa \opt_\I + \eta), (1+\frac{1}{2\lambda}) (\lceil\log(n)\rceil+1) \opt_\I \},
	\]
	with $\kappa = 2$. Observe that, when considering tour predictions, $\fp_\I = \opt_\I + \eta$. Thus, we obtain the same bound on $\rob(\fp,G)$ but with $\kappa = 1$. This concludes the proof of \Cref{theorem_learningaugmented}.
	\subsection{PAC learnability of the predictions}
	\label{sec_pac_learnability}

	We now discuss the learnability of tree and tour predictions. To allow for predicted tours or trees that are consistent with the graph to be explored, we assume that the set of $n$ (labeled) vertices is fixed and the graph~$G$ is complete. This may seem like a strong restriction of general inputs to the graph exploration problem. However, notice that the cost~$c(e)$ of an edge~$e$ is still only revealed when the first endpoint of~$e$ is explored. There is no improved online algorithm known for this~special~case.

	Firstly, we show PAC learnability of tree predictions.
	Our goal is to predict a spanning tree in~$G$ of low expected cost
	when  
	edge costs are drawn randomly from an unknown distribution~\ud.
	We assume that we can sample cost vectors $c$ efficiently and i.i.d.\ from~$\ud$
	to obtain a training set.
	Denote by~$\hs$ the set of all labeled spanning trees in~$G$, and, for each~$T \in \hs$, let~$\eta(T,c) = c(T) - c(T^*)$ denote the error of~$T$ with respect to the edge costs~$c$, where~$T^*$ is an MST of~$G$ with respect to the edge costs~$c$. %
	As~$c$ is drawn randomly from~$\ud$, the value~$\eta(T,c)$ is a random variable. 
	Our goal is to learn a prediction~$T_p \in \hs$ that (approximately) minimizes the expected error~$\EX_{c \sim  \ud}[\eta(T,c)]$ over all~$T \in \hs$. 

	We show that there is an efficient learning algorithm that determines a tree prediction that has nearly optimal expected cost with high probability and has a sample size polynomial in $n$ and $\eta_{\max}$, an upper bound on~$\eta(T,c)$. The existence and value of such an upper bound depends on the unknown distribution~$\ud$. %
	Thus, to select the correct value of~$\eta_{\max}$ when determining the training set size, we require such minimal prior knowledge of~$D$, which does not seem unreasonable in applications.

	\begin{theorem}
		\label{Theorem_learnability}
		Let $\eta_{\max}$ be an upper bound on~$\eta(T,c)$ and $\bar{T} = \arg\min_{T\in \hs} \EX_{c\sim \ud}[\eta(T,c)]$.
		Under the assumptions above \jnew{(in particular, that the graph is complete and has a fixed number $n$ of vertices)}, and for any~$\eps, \delta \in (0,1)$, there exists a learning algorithm that returns a~$T_p \in \hs$ such that $\EX_{c \sim \ud}[\eta(T_p,c)] \le \EX_{c \sim \ud}[\eta(\bar{T},c)] + \eps$ with probability at least~$(1-\delta)$. It does so using a training set of size~$m \in \mathcal{O}\big(\frac{(n \cdot \log n - \log\delta)\cdot \eta_{\max}^2}{\eps^2}\big)$ and in time polynomial in~$n$ and~$m$.
	\end{theorem}

	\begin{proof}
		We show how to adapt the classical \emph{Empirical Risk Minimization (ERM)} algorithm~(see, e.g.,~\cite{Vapnik1992}). %
		ERM first i.i.d.~samples a trainingset~$S=\{c_1,\ldots,c_m\}$ of~$m$ edge weight vectors from~$\ud$.
		Then, it returns a tree~$T_p \in \hs$ that minimizes the \emph{empirical error}~$\eta_S(T_p) = \frac{1}{m} \sum_{i=1}^{m} \eta(T_p,c_i)$.
		
		Recall that $\hs$ is the set of all labeled spanning trees in a complete graph of~$n$ vertices. By Cayley's formula~\cite{Cayley}, %
		$|\hs| = n^{n-2}$.
		Since~$\hs$ is finite, it satisfies the \emph{uniform convergence property} (cf.~\cite{Shalev2014}, or observe that~$\hs$ 
		has finite VC-dimension and cf.~\cite{Vapnik1992}). 
		Given a sample of size 
		\[
			m = \left\lceil \frac{2\log(2|\hs|/\delta)\eta_{\max}^2}{\eps^2} \right\rceil \in \mathcal{O}\left(\frac{(n \cdot \log(n) - \log(\delta))\cdot \eta_{\max}^2}{\eps^2}\right),
		\]
		it holds that~$\EX_{c \sim D}[\eta(T_p,c)] \le \EX_{c \sim D}[\eta(\bar{T},c)] + \eps$ with probability at least~$(1-\delta)$, where~$T_p$ is the spanning tree learned by ERM (cf.~\cite{Shalev2014,vapnik1999}). 
		It remains to bound the running time of ERM.
		
		While the sampling stage of ERM is polynomial in~$m$, naively computing a~$T_p \in \hs$ that minimizes the empirical error~$\eta_S(T_p)$ might be problematic because of~$\hs$'s exponential size.
		However, we show that %
		computing~$T_p$ corresponds to an MST problem, %
		well-known to be solvable in polynomial time. %
		
		Denote by $T_i^*$ an MST for 
		cost vector~$c_i$.
		By definition of~$\eta(T,c)$, we can write the empirical error as
		\[
		\eta_S(T) = \frac{1}{m} \left(\sum_{i=1}^m c_i(T) - \sum_{i=1}^m c_i(T_i^*)\right),
		\]
		for each~$T\in \hs$.
		Since~$\frac{1}{m}$ and~$\sum_{i=1}^m c_i(T_i^*)$ are independent of~$T$, the problem of computing a~$T\in \hs$ that minimizes~$\eta_S(T)$ corresponds to minimizing~$\eta_S'(T) = \sum_{i=1}^m c_i(T) = \sum_{e\in T} \sum_{i=1}^m c_i(e)$.
		This implies~$\eta'_S(T) = c_S(T)$, where the edge weights~$c_S$ are defined as~$c_S(e) = \sum_{i=1}^m c_i(e)$ for~$e \in E$. Hence, the problem reduces to computing an MST for edge weights~$c_S$.
	\end{proof}
	
	We show a similar result for learning a predicted tour. 
	Again, assume that $G=(V,E)$ is complete and has a fixed number $n$ of vertices. 
	Let $\ths$ be the set of all tours, i.e., the set of all permutations of $V$.
	Let~$\eta(\tau,c) = c(\tau) - c(\tau^*)$, where $\tau \in \ths$, $\tau^*$ is an optimal tour w.r.t.~cost vector $c$, and $c(\tau)$ is the cost of tour $\tau$ assuming that the next vertex $v$ in $\tau$ is always visited via a shortest path in the graph induced by the previous vertices in $\tau$ and $v$.
	Our goal is to learn a predicted tour~$\tau_p \in \ths$ that (approximately) minimizes the expected error~$\EX_{c \sim  \ud}[\eta(\tau,c)]$ over all~$\tau \in \ths$. 
	As $|\ths| = n! \in \mathcal{O}(n^n)$, and assuming an upper bound $\eta_{\max}$ on $\eta(\tau,c)$, we apply ERM with the same sample complexity as in~\Cref{Theorem_learnability}.
	However, the problem of computing a $\tau \in \ths$ that minimizes the empirical error in trainingset~$S$ contains TSP.
	Thus, %
	unless~$P \not= NP$, we settle for an exponential running time. 
	
	\begin{theorem}
		\label{Theorem_learnability_tours_1}
		Let $\eta_{\max}$ be an upper bound on~$\eta(\tau,c)$ and $\bar{\tau} = \arg\min_{\tau\in \ths} \EX_{c\sim \ud}[\eta(\tau,c)]$.
		Under the assumptions above \jnew{(in particular, that the graph is complete and has a fixed number $n$ of vertices)}, and for any~$\eps, \delta \in (0,1)$, there exists a learning algorithm that returns a~$\tau_p \in \ths$ such that
		$\EX_{c \sim \ud}[\eta(\tau_p,c)] \le \EX_{c \sim \ud}[\eta(\bar{\tau},c)] + \eps$ with probability at least~$(1-\delta)$. 
		It does so using a training set of size~$m \in \mathcal{O}(\frac{(n \cdot \log n - \log\delta)\cdot \eta_{\max}^2}{\eps^2})$ and in time polynomial in $m$ but exponential in $n$.
	\end{theorem}
	
	\begin{proof}
		Since $|\ths| = n! \in \mathcal{O}(n^n)$, and assuming an upper bound $\eta_{\max}$ on $\eta(\tau,c)$, we again apply ERM with the same sample complexity as in~\Cref{Theorem_learnability}.
		
		Using the same argumentation as in the proof of~\Cref{Theorem_learnability}, we can show that,
		for a training set $S = \{c_1,\ldots,c_m\}$ with $m \in \mathcal{O}(\frac{(n \cdot \log n - \log\delta)\cdot \eta_{\max}^2}{\eps^2})$, it holds $\EX_{c \sim \ud}[\eta(\tau_p,c)] \le \EX_{c \sim \ud}[\eta(\bar{\tau},c)] + \eps$ with probability at least~$(1-\delta)$, where $\tau_p$ is the element of $\ths$ that minimizes the empirical error $\eta_S(\tau_p) = \frac{1}{m} \sum_{i=1}^m \eta(\tau_p,c_i)$.
		Thus, it remains to argue about the computation of the $\tau_p \in \ths$ that minimizes 
		\[
			\eta_S(\tau_p) = \frac{1}{m} \sum_{i=1}^m \eta(\tau_p,c_i) = \frac{1}{m} \left(\sum_{i=1}^m c_i(\tau_p) - \sum_{i=1}^m c_i(\tau_i^*) \right),
		\]
		where $\tau_i^*$ is the optimal tour for cost vector $c_i$.
		Since $\frac{1}{m}$ and $\sum_{i=1}^m c_i(\tau_i^*)$ are independent of the actual predicted tour $\tau_p$, a tour $\tau_p$ minimizes the empirical error if and only if it minimizes $\sum_{i=1}^m c_i(\tau_p)$.
		Since the problem of computing this tour contains TSP, it is NP-hard and we settle for an exponential running time. 
		We can find $\tau_p$ by computing $\sum_{i=1}^m c_i(\tau)$ for each $\tau \in \ths$.
		This can be done in running time $\mathcal{O}(m \cdot |\ths|) = \mathcal{O}(m \cdot n^n)$ by iterating through all elements of $S \times \ths$.
	\end{proof}

	\section{Experimental analysis}
	\label{sec_experiments}

	We present the main results of our empirical experiments and discuss their significance with respect to our algorithms' performance.
	
	We analyze the performance of the robustification scheme for various instances, namely, real world city road networks, symmetric graphs of the TSPlib library~\cite{Reinelt91,tsplib}, %
	and special artificially generated graphs. 
	We use the \emph{empirical competitive ratio} as performance measure; for an algorithm $\alg$, it is defined as the average of the ratio $\alg_\I/\opt_\I$ over all input instances $I$ in our experiments. Since the offline optimum~$\opt_\I$ is the optimal TSP tour which is NP-hard to compute, we lower bound this value by the cost of an MST for instance $\I$, which we can compute efficiently. This leads to larger empirical competitive ratios, but the relative differences between any two algorithms remains the same.	
	
	To evaluate learning-augmented algorithms, we compute a (near) perfect prediction and iteratively worsen it to get further predictions. Again, due to the intractability of the underlying TSP problem, we use heuristics to determine a ``perfect'' prediction, namely Christofides' algorithm~\cite{christofides1976} and 2-opt~\cite{croes1958two-opt}. Such weaker input disfavors our algorithms as having better predictions can only improve the performance of our learning-augmented algorithms. 
	\anew{The \emph{relative prediction error} is defined as the ratio between the prediction error and the cost of an MST for the instance.}

	For the experiments, we consider the classical exploration algorithms depth first search (\dfsalg), nearest neighbor (\nn), and hierarchical depth first search (\hdfs), as well as the constant-competitive algorithm for graphs of bounded genus, including planar graphs, called \block~\cite{KalyanasundaramP93,MegowMS12}. Regarding learning-augmented algorithms we look at the algorithm that follows a prediction (\fp). 
	We denote by $\overline{\rob}(\alg,\lambda)$ the modified robustification scheme with parameter $\lambda > 0$ applied to an algorithm $\alg$.%
	
	\subsection{Implementation details}
	The simulation software is written in Rust (2018 edition) and available on GitHub\footnote{\url{https://github.com/Mountlex/graph-exploration}}. There are also instructions to reproduce all presented experiments, as well as generated predictions and simulation outputs which were used to create the figures in this document. 
	We executed all experiments in Ubuntu 18.04.5 on a machine with two AMD EPYC ROME 7542 CPUs (64 cores in total) and 1.96 TB RAM.
	
	\paragraph{General implementation details}
	We first note that edge costs are represented as unsigned integers.
	We use Prim's algorithm to compute an MST of a graph. For computing a minimum cost perfect matching in a general graph, as required in Christofides' algorithm, we use the \emph{Blossom V}~\cite{kolmogorov09} implementation of a modified variant of Edmonds' Blossom algorithm to solve this problem efficiently. Regarding the exploration algorithms, the implementation of \dfsalg always explores the cheapest unexplored edge first. 
	The nodes in the input graphs are indexed by integers, and all algorithms use this total order to break ties. The exploration always starts at the vertex with index~$0$.
	
	\paragraph{Prediction Generation}
	We compute perfect predictions using Christofides' algorithm~\cite{christofides1976} and further improve these by the 2-opt local search heuristic~\cite{croes1958two-opt}.
	Given such a ``perfect'' prediction, we generate worse predictions by iteratively increasing its total cost using the reversed 2-opt algorithm. 
	That is, we reverse a subsequence of length at least~$2$ such that this reversion increases the total cost.
	
	\subsection{Experimental results}
	
	In the following, we describe the results of our experiments.
	
	\paragraph{Comparison of robustification schemes} 
	For our experiments, we use the modified robustification scheme over the basic scheme. This is because we observed that the former performs overall better in our experiments although it has a slightly worse theoretic robustness guarantee. We present the results for both variants for the learning-augmented setting for the city of Stockholm, which is the largest city graph we consider, in \Cref{fig_robustifaction_comparison}. 
	The basic robustification scheme does not improve over \nn even for small prediction errors, whereas the modified variant gives a significantly better performance for this case. For large errors, the modified robustification scheme with~$\lambda=0.5$ does indeed perform worse than the basic scheme. For~$\lambda = 0.75$, it performs as good as the basic scheme for large errors, while it still improves upon \nn for the case of good predictions.
	
	\begin{figure}[tb]
		\begin{subfigure}[t]{0.45\textwidth}%
			\includegraphics[width=\textwidth]{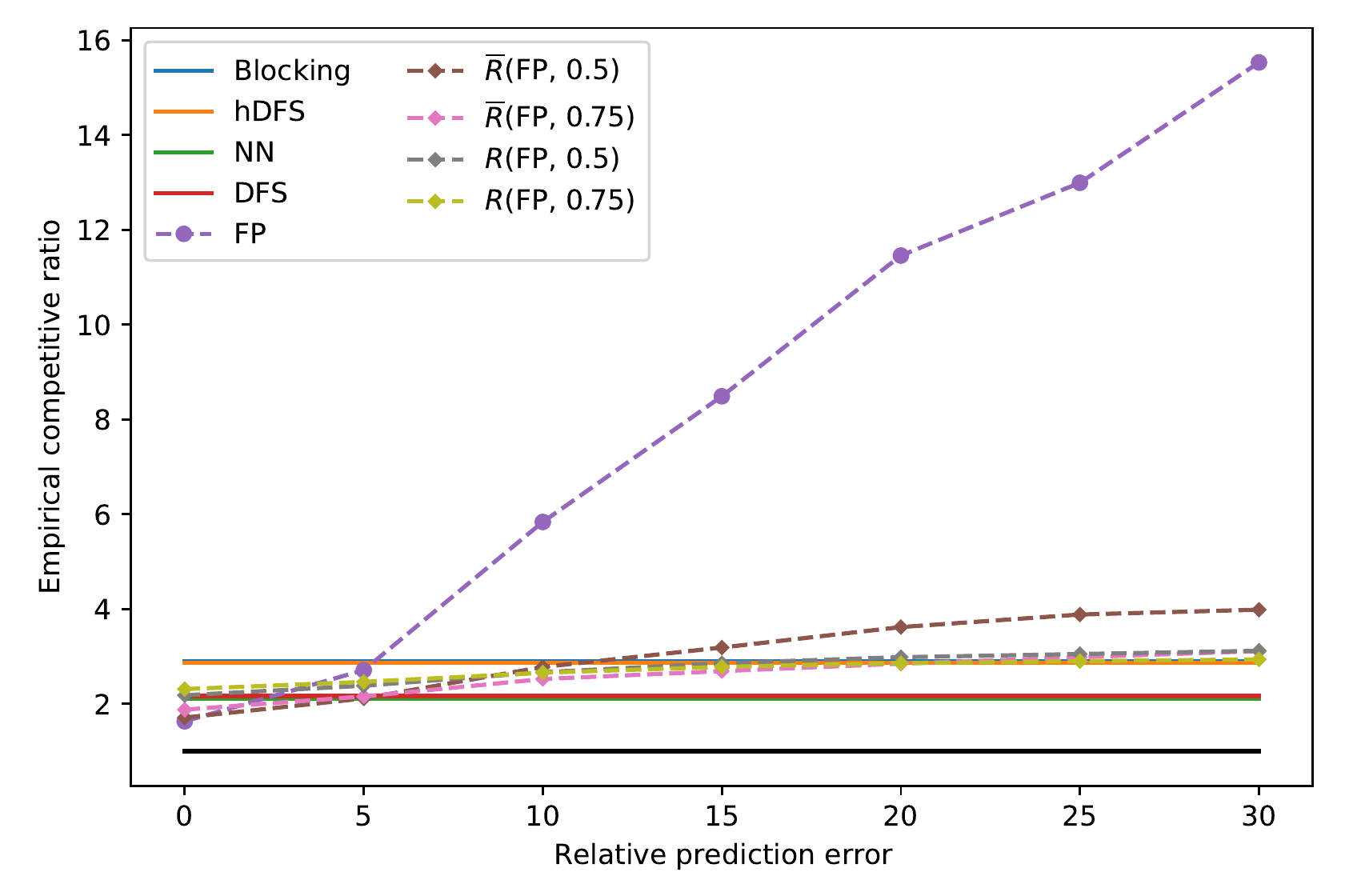}%
			\caption{Whole picture: results for large relative errors.}%
		\end{subfigure}\hfill%
		\begin{subfigure}[t]{0.45\textwidth}%
			\includegraphics[width=\textwidth]{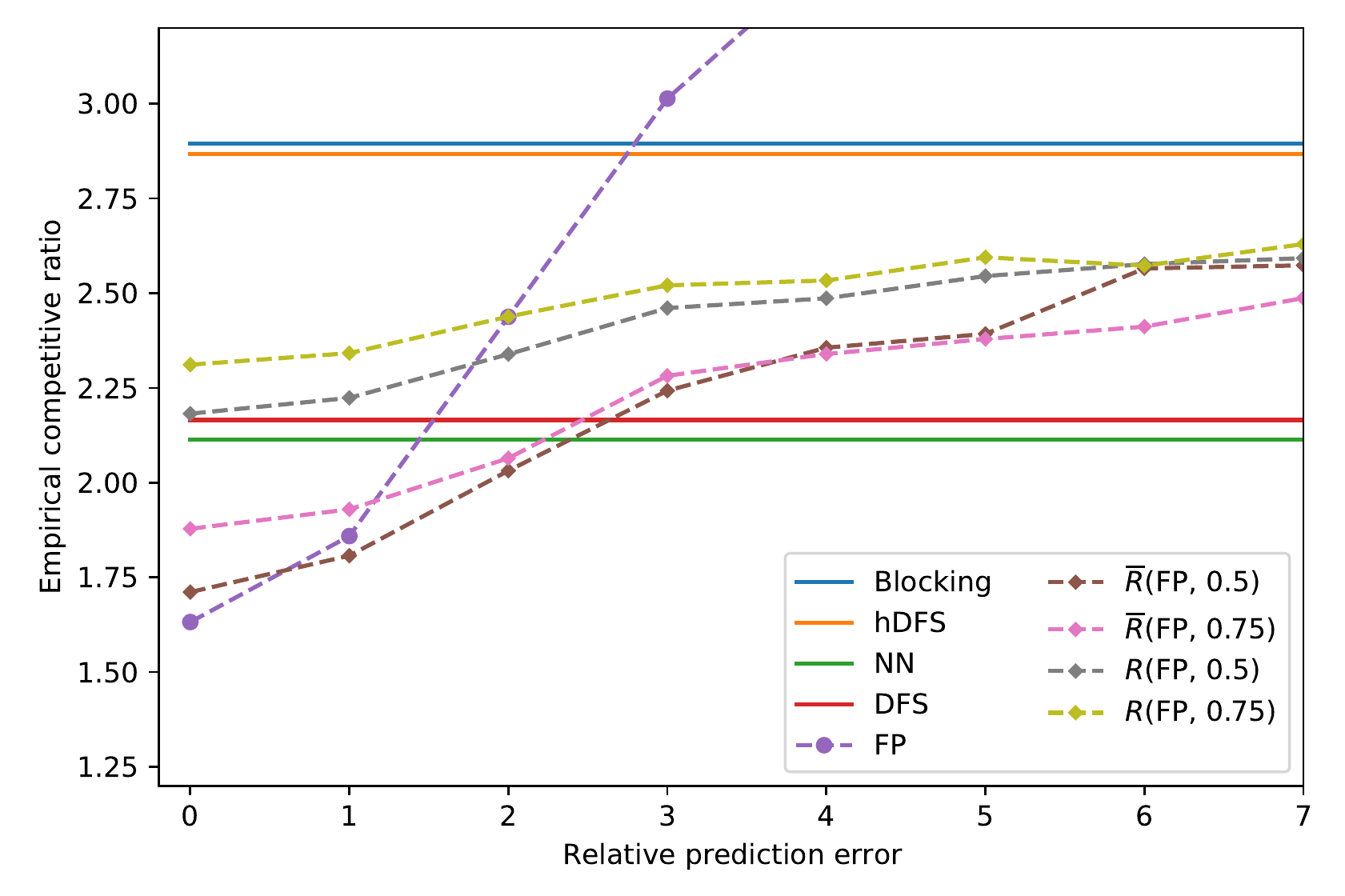}%
			\caption{Zoomed-in: results for small relative errors.}%
		\end{subfigure}
		\caption{Comparison between both robustification variants for the city of Stockholm.}
		\label{fig_robustifaction_comparison}
	\end{figure}
	
	\paragraph{TSPLib instances}
	We consider \anew{the 72} graphs of the %
	TSPlib library with at most 600 nodes (for performance reasons) and evaluate the classical exploration algorithms %
	as well as their robustified variants.
	We construct instances directly from the provided XML files and round all edge weights to integers.
	The results are displayed in \Cref{fig_tsplib}. 
	Observe that \nn outperforms \hdfs and \block. %
	While for small values of $\lambda$ the performance of the robustified algorithms stays close to that of their base variants, %
	it improves quickly with an increasing %
	$\lambda$ and eventually converges to that of \nn. This illustrates that if \nn performs well, our robustification scheme exploits this and improves algorithms %
	performing worse.
	Note that %
	TSPlib provides complete graphs and our implementation of \dfsalg explores a closest unexplored child first. Thus, \dfsalg and \nn act identically.
	We display the standard deviations of this experiment in~\Cref{tab_stddev_tsplib}.
	
	\begin{table}[b]
		\caption{Observed standard deviations in TSPLib experiments (\Cref{fig_tsplib}).}
		\label{tab_stddev_tsplib}
		\centering
		\begin{tabular}{lcc}
			\toprule
			\multirow{2}{*}{Algorithm} & \multicolumn{2}{c}{Standard deviation} \\ \cmidrule(lr){2-3}
			& min & max \\
			\midrule	
			\block & \multicolumn{2}{c}{0.258324} \\
			\hdfs & \multicolumn{2}{c}{0.228537} \\
			\nn & \multicolumn{2}{c}{0.178597} \\
			\dfsalg & \multicolumn{2}{c}{0.178115} \\
			\rhdfs & 0.17447 & 0.230531 \\
			\rblock & 0.17162 & 0.258324 \\
			\bottomrule 
		\end{tabular}	
	\end{table}
	
	\begin{figure}[t]
			\centering
			\includegraphics[scale=0.5]{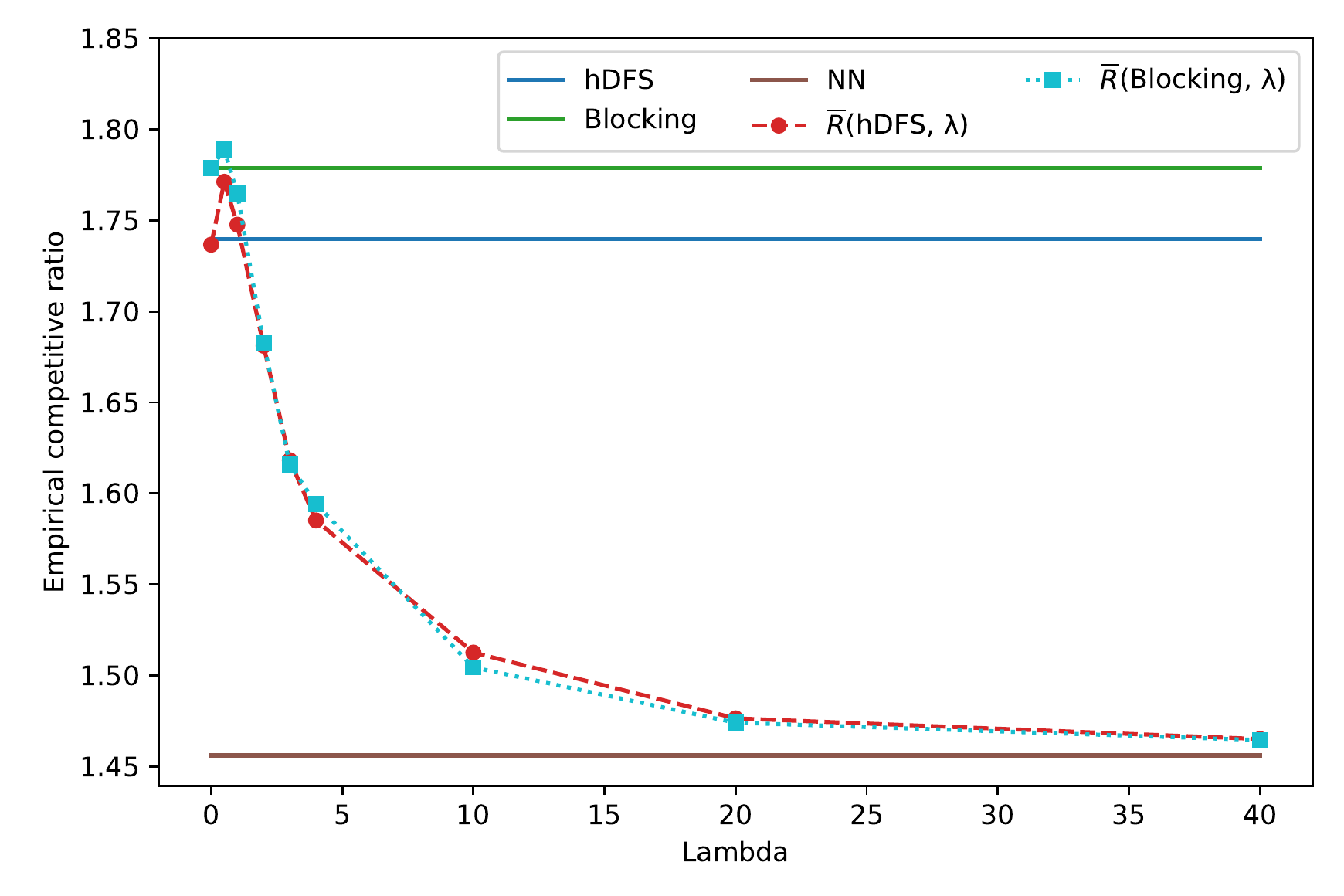}
			\caption{Results for TSPLib instances.}
			\label{fig_tsplib}
	\end{figure}
	
	\paragraph{Rosenkrantz graphs}
	This experiment looks at graphs on which \nn is known to perform badly. 
	Specifically, we consider a family of graphs that are artificially constructed in~\cite{RosenkrantzSL13} in order to show a lower bound of~$\Omega(\log n)$ on the competitive ratio. 
	Each member of the family corresponds to a size parameter~$i$ and consists of $n = \Theta(2^i)$ nodes. 
	The cost of a \nn tour increases linearly with~$i$. 
	Forcing \nn to \lnew{incur a} large exploration cost in these graphs requires a specific tie breaking~\cite{RosenkrantzSL13}. We accomplish this by scaling up all edge costs and then breaking ties by slightly modifing these costs appropriately.
	We refer to this family as Rosenkrantz graphs. 
	There exist variations of the Rosenkrantz construction, that suggest that we can %
	expect similar experimental results,
	even for Euclidean graphs and unit weights~\cite{HurkensW04}.
	Besides \nn, we consider the algorithms \block, \hdfs and \fp (relative error of~$5$) with robustification parameters 0, 1, and 20, respectively.
	Again, \dfsalg acts like \nn on these graphs. 

	The results (\Cref{fig_rosenkrantz_full})
	show that the slight robustification %
	\rfp[1] improves \fp's performance significantly. 
	This remains true for large $i$, even though \nn is performing increasingly bad here. 
	If we increase the amount of robustification, %
	i.e., \rfp[20], the slope is equal to \nn's slope, but %
	it still outperforms \nn and \fp.
	Surprisingly, for \hdfs, this drawback does not appear: \rhdfs[20] does indeed perform worse than for smaller~$\lambda$'s, but its competitive ratio does not grow as~$i$ increases. For \rhdfs[1] there is almost no drop in performance when compared to \hdfs. 
	
	\begin{figure}[tb]
		\begin{subfigure}[t]{0.45\textwidth}
			\includegraphics[width=\textwidth]{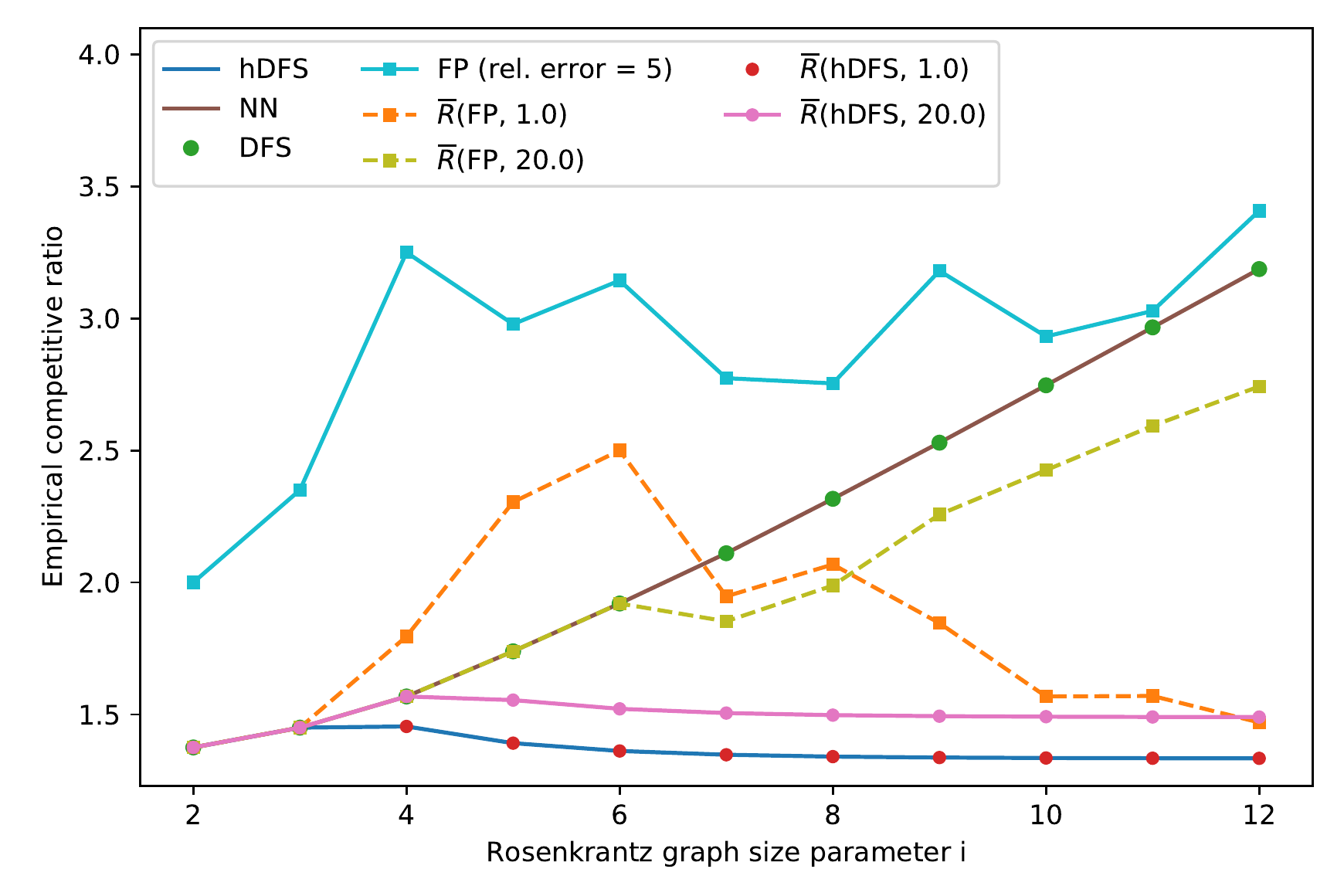}      
			\caption{Results for \hdfs and \fp on Rosenkrantz graphs.}
		\end{subfigure}\hfill
		\begin{subfigure}[t]{0.45\textwidth}
			\includegraphics[width=\textwidth]{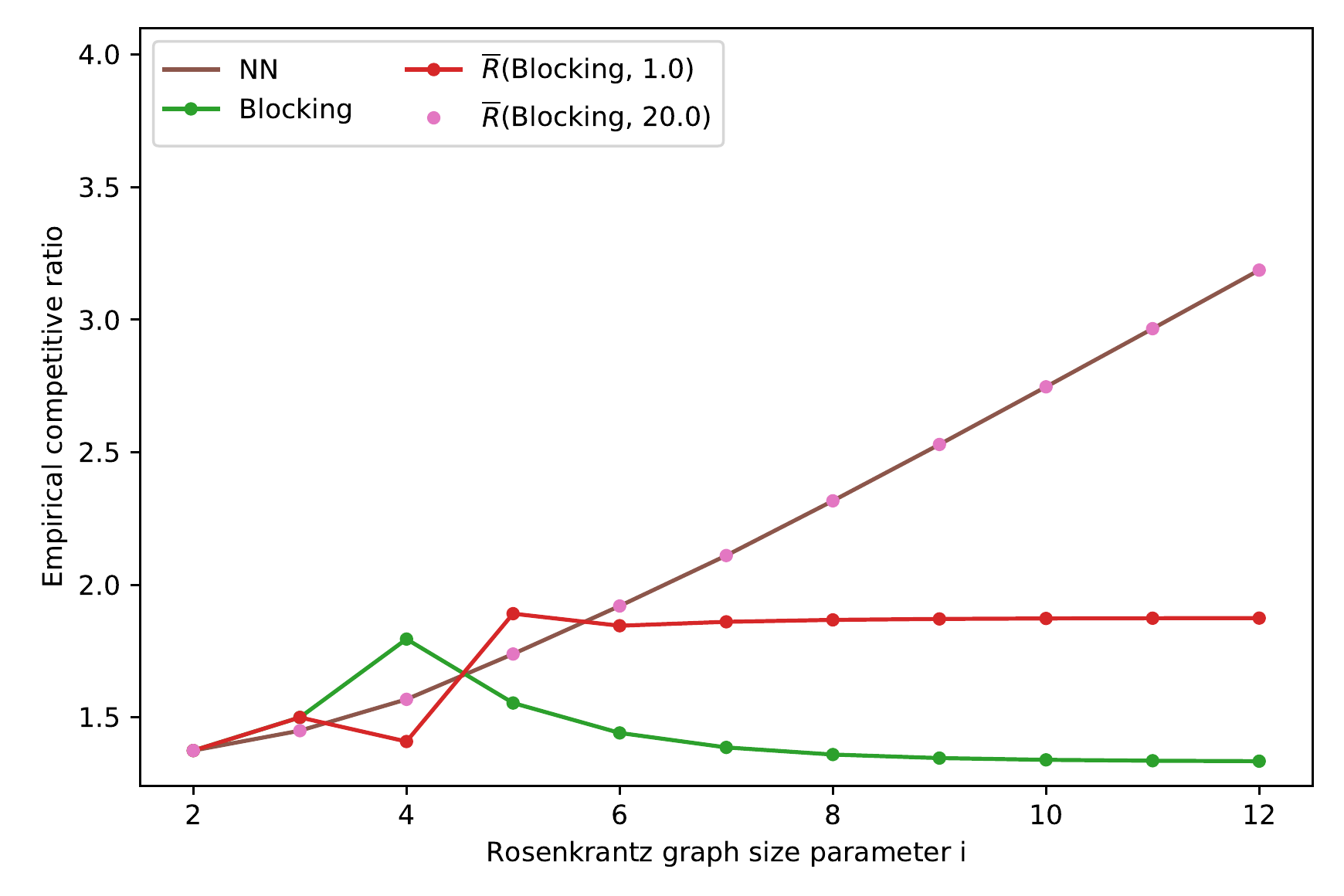}      
			\caption{Results for \block on Rosenkrantz graphs.}
		\end{subfigure}
		\caption{The full set of results for Rosenkrantz graphs.}
		\label{fig_rosenkrantz_full}
	\end{figure}

	In summary, we have indications that even a slight robustification clearly improves algorithms that otherwise perform badly without the performance degrading too much when \nn performs poorly. 
	Even more interestingly, %
	these experiments actually show that the robustification scheme applied to an online algorithm robustifies \nn as well. %
	Since \nn generally performs %
	notably well, c.f. \Cref{fig_tsplib}, this may be %
	useful in practice as protection against unlikely but possible bad scenarios for \nn; in particular, in safety relevant applications where solutions have to satisfy %
	strict performance bounds.
	
	\paragraph{City road networks}
	Finally, we provide experiments to evaluate our learning-augmented algorithm in the context of the real-world task of exploring a city road network. 
	To this end, we consider the ten largest (by population) capitals in Europe with a population less than one million.
	\begin{figure}[b]
		\begin{subfigure}[t]{0.45\textwidth}
			\includegraphics[width=\textwidth]{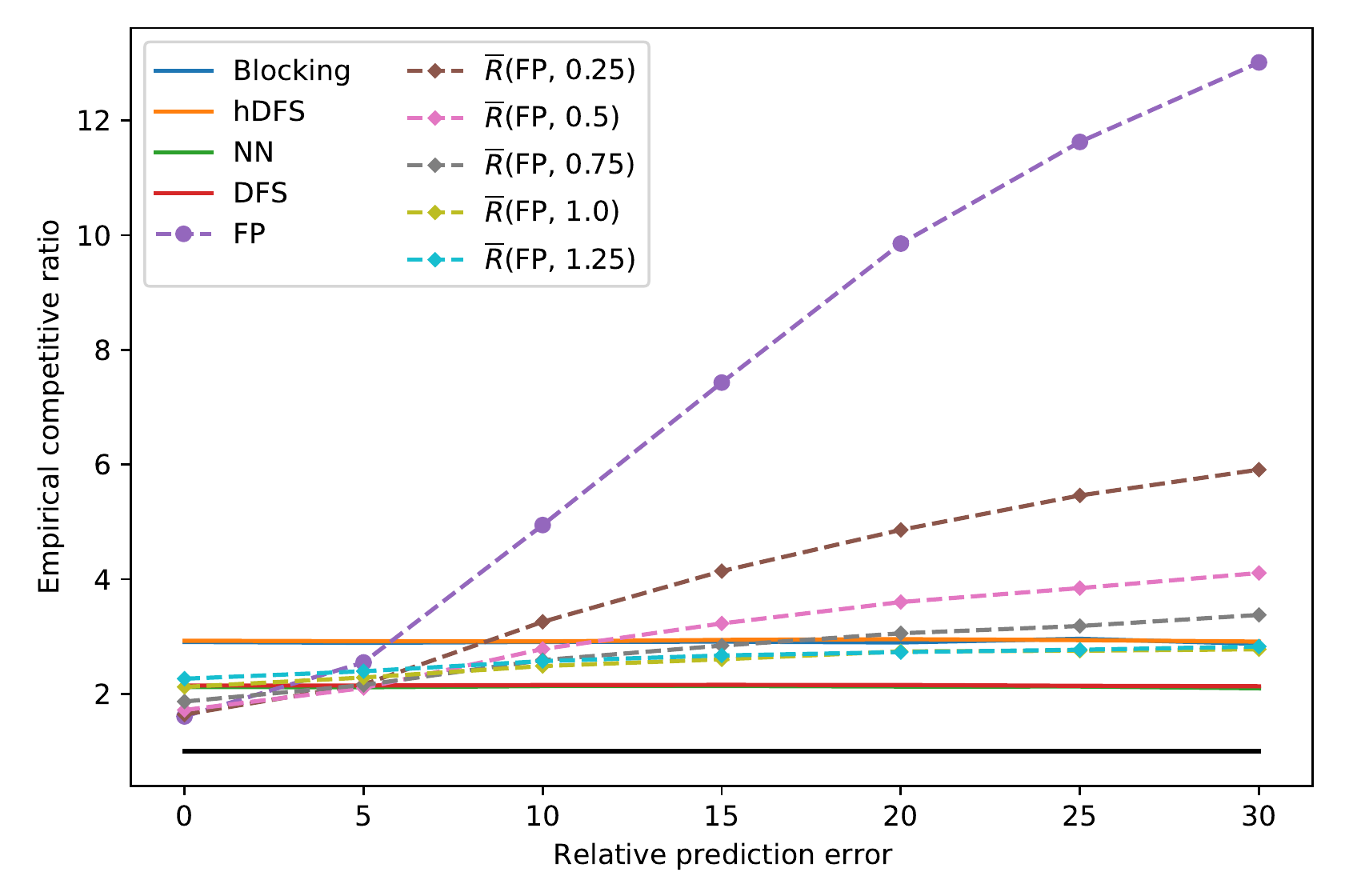}
			\caption{Whole picture: results for large relative errors.}
			\label{fig_cities}
		\end{subfigure}\hfill
		\begin{subfigure}[t]{0.45\textwidth}
			\includegraphics[width=\textwidth]{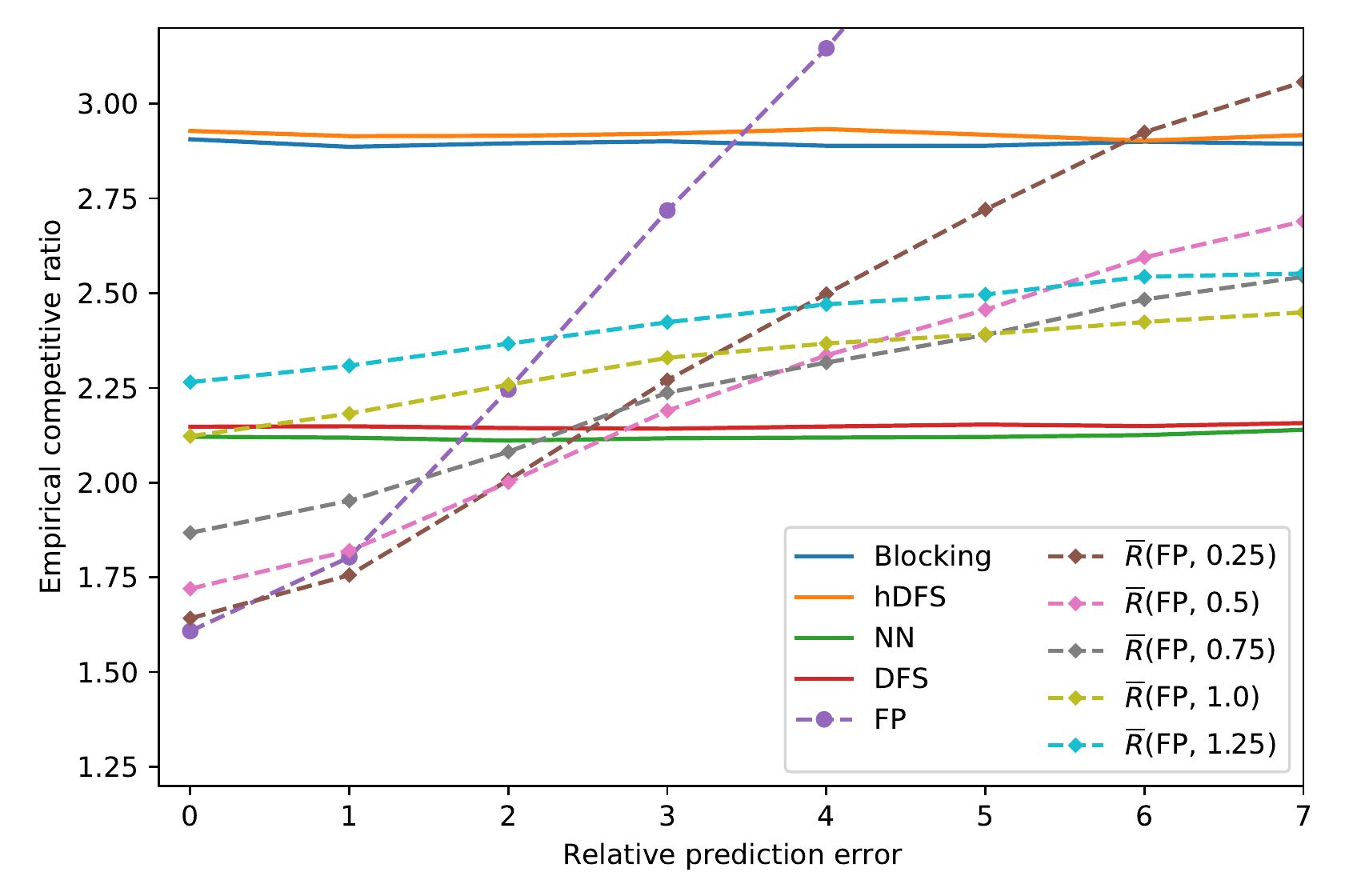}
			\caption{Zoomed-in: results for small relative errors.}
			\label{fig_cities_close}
		\end{subfigure}
		\caption{Average performance of classical and learning-augmented algorithms on city networks.}
	\end{figure}

	Our instances represent the road networks of these cities, built with OSMnx~\cite{boeing17osmnx} from OpenStreetMap data~\cite{osm}.\footnote{We downloaded the city graphs on 11.05.2021.}
	We used the name of the city and the network type \texttt{drive} as input. 
	The sizes of the resulting city graphs are displayed in~\Cref{tab_cities}.
	The observed standard deviations in this experiment are given in~\Cref{tab_stddev_cities}.
	The generated graphs are connected.

	For each instance, we generate 150 predictions with relative errors ranging from 0 up to 30. 
	The average results %
	(\Cref{fig_cities,fig_cities_close}) %
	indicate that, for relative %
	errors less than $2.5$, we improve upon the best performing classical %
	algorithms \nn and \dfsalg by using \rfp with $\lambda < 1$, %
	while the increase %
	of \fp for large errors is significantly smaller after robustification. 
	Moreover, \rfp[0.5] and \rfp[0.75] perform only slightly worse than \hdfs and \block for large relative %
	errors.
	
	\begin{table}[tb]
		\caption{Graph sizes of the considered city road networks.}
		\label{tab_cities}
		\centering
		\begin{tabular}{lcc}
			\toprule
			City &	\# nodes & \# edges \\
			\midrule
			Stockholm & 13029 & 17836 \\
			Amsterdam & 11652 & 17161 \\
			Zagreb & 11873 & 15214 \\
			Oslo & 8225 & 11441 \\
			Chisinau & 3097 & 4348 \\
			Athens & 4880 & 6518 \\
			Helsinki & 9607 & 13259 \\
			Copenhagen & 7002 & 10224 \\
			Riga & 8535 & 11847 \\
			Vilnius & 6961 & 9226 \\
			\bottomrule
		\end{tabular}	
	\end{table}

	We conclude that, given predictions of somewhat reasonable quality, it is possible to beat the best known online algorithms in terms of solution quality, while still providing the security that, even if some predictions turn out to be bad, the consequences are not too harsh.
	While ``somewhat reasonable'' appears to be a relative error of roughly $2.5$, recall that our perfect predictions are only approximate tours, which may be a constant factor away from the optimal tour. 
	With logistic companies in mind, where margins are \nnew{tight} and every potential for optimization needs to be taken advantage of (while still making sure that trucks do arrive eventually), this \nnew{seems} to be a potentially useful outcome.
	
	\begin{table*}[bt]
		\caption{Observed standard deviations in city road network experiments.}
		\label{tab_stddev_cities}
		\centering
		\begin{tabular}{lcccc}
			\toprule
			\multirow{3}{*}{Algorithm} & \multicolumn{4}{c}{Standard deviations over relative error ranges of size} \\ \cmidrule(lr){2-5}
			& \multicolumn{2}{c}{5~(\Cref{fig_cities})} & \multicolumn{2}{c}{1~(\Cref{fig_cities_close})} \\ \cmidrule(lr){2-3} \cmidrule(lr){4-5}
			& min & max & min & max \\
			\midrule	
			\block & 0.136003 & 0.233786 & 0.123958 & 0.252672 \\
			\hdfs & 0.080334 & 0.108376 & 0.070152 & 0.11898 \\
			\nn & 0.05486 & 0.067854 & 0.050882 & 0.077972 \\
			\dfsalg & 0.048957 & 0.057941 &  0.041946 & 0.062405 \\
			\fp & 0.046025 & 2.019176 & 0.046025 & 2.58974 \\
			\rfp[0.25] & 0.046669 & 0.822972 & 0.046669 &  0.879679 \\
			\rfp[0.5] & 0.042480 & 0.463836 & 0.042480 &  0.533792 \\
			\rfp[0.75] & 0.049568 & 0.397593 & 0.049568 & 0.443789 \\
			\rfp[1.0] & 0.071852 & 0.235454 & 0.065083 & 0.260136 \\
			\rfp[1.25] & 0.078526 & 0.127131 & 0.057414 & 0.140951 \\
			\bottomrule 
		\end{tabular}	
	\end{table*}

	\section{Conclusion}
	
	We initiate the study of learning-augmented algorithms for the classical online graph exploration problem. By carefully interpolating between the algorithm that blindly follows any given prediction and Nearest Neighbor, we are able to give a learning-augmented online algorithm whose theoretical worst-case bound linearly depends on the prediction error. %
	In particular, if the prediction is close to perfect, this substantially improves upon %
	any known online algorithm without sacrificing the worst-case bound.
	We complement these theoretical results by computational experiments on various instances, ranging from symmetric graphs of the TSPLib library and Rosenkrantz graphs to city road networks. 
	Moreover, we design a framework to robustify any given online algorithm by carefully interpolating between this algorithm and Nearest Neighbor. %
	This is potentially very interesting also in the area of stochastic optimization or when designing data-driven algorithms, that typically provide probabilistic guarantees but may perform very poorly in the worst case. It remains open whether online graph exploration (without additional information) allows for any constant-competitive~algorithm.

	\small
	\bibliographystyle{abbrv}
	\bibliography{explore.bib}

\begin{thebibliography}{10}

\bibitem{AngelopoulosDJKR20}
S.~Angelopoulos, C.~D{\"{u}}rr, S.~Jin, S.~Kamali, and M.~P. Renault.
\newblock Online computation with untrusted advice.
\newblock In {\em {ITCS}}, volume 151 of {\em LIPIcs}, pages 52:1--52:15.
  Schloss Dagstuhl - Leibniz-Zentrum f{\"{u}}r Informatik, 2020.

\bibitem{AntoniadisCE0S20}
A.~Antoniadis, C.~Coester, M.~Eli{\'{a}}s, A.~Polak, and B.~Simon.
\newblock Online metric algorithms with untrusted predictions.
\newblock In {\em {ICML}}, volume 119 of {\em Proceedings of Machine Learning
  Research}, pages 345--355. {PMLR}, 2020.

\bibitem{ApplegateBCC2006-book}
D.~L. Applegate, R.~E. Bixby, V.~Chvatál, and W.~J. Cook.
\newblock {\em The Traveling Salesman Problem: A Computational Study}.
\newblock Princeton University Press, 2006.

\bibitem{Azar1993}
Y.~Azar, A.~Z. Broder, and M.~S. Manasse.
\newblock On-line choice of on-line algorithms.
\newblock In {\em {SODA}}, pages 432--440. {ACM/SIAM}, 1993.

\bibitem{AzarLT2021}
Y.~Azar, S.~Leonardi, and N.~Touitou.
\newblock Flow time scheduling with uncertain processing time.
\newblock In {\em {STOC}}, pages 1070--1080. {ACM}, 2021.

\bibitem{Balcan2021}
M.~Balcan.
\newblock Data-driven algorithm design.
\newblock In {\em Beyond the Worst-Case Analysis of Algorithms}, pages
  626--645. Cambridge University Press, 2020.

\bibitem{Balcan19}
M.~Balcan, D.~F. DeBlasio, T.~Dick, C.~Kingsford, T.~Sandholm, and E.~Vitercik.
\newblock How much data is sufficient to learn high-performing algorithms?
  generalization guarantees for data-driven algorithm design.
\newblock In {\em {STOC}}, pages 919--932. {ACM}, 2021.

\bibitem{Balcan18-1}
M.~Balcan, T.~Dick, T.~Sandholm, and E.~Vitercik.
\newblock Learning to branch.
\newblock In {\em {ICML}}, volume~80 of {\em Proceedings of Machine Learning
  Research}, pages 353--362. {PMLR}, 2018.

\bibitem{Balcan2018}
M.~Balcan, T.~Dick, and E.~Vitercik.
\newblock Dispersion for data-driven algorithm design, online learning, and
  private optimization.
\newblock In {\em {FOCS}}, pages 603--614. {IEEE} Computer Society, 2018.

\bibitem{Balcan18-2}
M.~Balcan, T.~Dick, and C.~White.
\newblock Data-driven clustering via parameterized lloyd's families.
\newblock In {\em NeurIPS}, pages 10664--10674, 2018.

\bibitem{BamasMRS20}
{\'{E}}.~Bamas, A.~Maggiori, L.~Rohwedder, and O.~Svensson.
\newblock Learning augmented energy minimization via speed scaling.
\newblock In {\em NeurIPS}, pages 15350--15359, 2020.

\bibitem{Barthelemy18}
M.~Barthelemy.
\newblock Spatial networks.
\newblock In {\em Encyclopedia of Social Network Analysis and Mining. 2nd Ed}.
  Springer, 2018.

\bibitem{Bello2017}
I.~Bello, H.~Pham, Q.~V. Le, M.~Norouzi, and S.~Bengio.
\newblock Neural combinatorial optimization with reinforcement learning.
\newblock In {\em {ICLR} (Workshop)}. OpenReview.net, 2017.

\bibitem{Berman96}
P.~Berman.
\newblock On-line searching and navigation.
\newblock In {\em Online Algorithms}, volume 1442 of {\em Lecture Notes in
  Computer Science}, pages 232--241. Springer, 1996.

\bibitem{Bhaskara20}
A.~Bhaskara, A.~Cutkosky, R.~Kumar, and M.~Purohit.
\newblock Online learning with imperfect hints.
\newblock In {\em {ICML}}, volume 119 of {\em Proceedings of Machine Learning
  Research}, pages 822--831. {PMLR}, 2020.

\bibitem{BirxDHK2021}
A.~Birx, Y.~Disser, A.~V. Hopp, and C.~Karousatou.
\newblock An improved lower bound for competitive graph exploration.
\newblock {\em Theor. Comput. Sci.}, 868:65--86, 2021.

\bibitem{Blum2000}
A.~Blum and C.~Burch.
\newblock On-line learning and the metrical task system problem.
\newblock {\em Mach. Learn.}, 39(1):35--58, 2000.

\bibitem{UngerFB18}
H.~B{\"{o}}ckenhauer, J.~Fuchs, and W.~Unger.
\newblock Exploring sparse graphs with advice.
\newblock In {\em {WAOA}}, volume 11312 of {\em Lecture Notes in Computer
  Science}, pages 102--117. Springer, 2018.

\bibitem{boeing17osmnx}
G.~Boeing.
\newblock Osmnx: New methods for acquiring, constructing, analyzing, and
  visualizing complex street networks.
\newblock {\em Comput. Environ. Urban Syst.}, 65:126--139, 2017.

\bibitem{Boeing2020}
G.~Boeing.
\newblock Planarity and street network representation in urban form analysis.
\newblock {\em Environment and Planning B: Urban Analytics and City Science},
  47(5):855--869, 2020.

\bibitem{BrandtFMW20}
S.~Brandt, K.~Foerster, J.~Maurer, and R.~Wattenhofer.
\newblock Online graph exploration on a restricted graph class: Optimal
  solutions for tadpole graphs.
\newblock {\em Theor. Comput. Sci.}, 839:176--185, 2020.

\bibitem{Cayley}
A.~Cayley.
\newblock A theorem on trees.
\newblock {\em Quart. J. Pure Appl. Math.}, 23:376--–378, 1889.

\bibitem{Chawla20}
S.~Chawla, E.~Gergatsouli, Y.~Teng, C.~Tzamos, and R.~Zhang.
\newblock Pandora's box with correlations: Learning and approximation.
\newblock In {\em {FOCS}}, pages 1214--1225. {IEEE}, 2020.

\bibitem{Chiotellis2020}
I.~Chiotellis and D.~Cremers.
\newblock Neural online graph exploration.
\newblock {\em CoRR}, abs/2012.03345, 2020.

\bibitem{christofides1976}
N.~Christofides.
\newblock Worst-case analysis of a new heuristic for the travelling salesman
  problem.
\newblock Technical Report 388, Graduate School of Industrial Administration,
  Carnegie Mellon University, 1976.

\bibitem{croes1958two-opt}
G.~A. Croes.
\newblock A method for solving traveling-salesman problems.
\newblock {\em Operations Research}, 6(6):791--812, 1958.

\bibitem{Dai2019}
H.~Dai, Y.~Li, C.~Wang, R.~Singh, P.~Huang, and P.~Kohli.
\newblock Learning transferable graph exploration.
\newblock In {\em NeurIPS}, pages 2514--2525, 2019.

\bibitem{DobrevKM2012}
S.~Dobrev, R.~Kr{\'{a}}lovic, and E.~Markou.
\newblock Online graph exploration with advice.
\newblock In {\em {SIROCCO}}, volume 7355 of {\em Lecture Notes in Computer
  Science}, pages 267--278. Springer, 2012.

\bibitem{DuttingLLV21}
P.~D{\"{u}}tting, S.~Lattanzi, R.~P. Leme, and S.~Vassilvitskii.
\newblock Secretaries with advice.
\newblock In {\em {EC}}, pages 409--429. {ACM}, 2021.

\bibitem{Elmiger2020}
J.~Elmiger, L.~Faber, P.~Khanchandani, O.~P. Richter, and R.~Wattenhofer.
\newblock Learning lower bounds for graph exploration with reinforcement
  learning.
\newblock In {\em Learning Meets Combinatorial Algorithms at NeurIPS2020},
  2020.

\bibitem{Fiat1991}
A.~Fiat, R.~M. Karp, M.~Luby, L.~A. McGeoch, D.~D. Sleator, and N.~E. Young.
\newblock Competitive paging algorithms.
\newblock {\em J. Algorithms}, 12(4):685--699, 1991.

\bibitem{Fritsch21}
R.~Fritsch.
\newblock Online graph exploration on trees, unicyclic graphs and cactus
  graphs.
\newblock {\em Inf. Process. Lett.}, 168:106096, 2021.

\bibitem{GasieniecR08}
L.~Gasieniec and T.~Radzik.
\newblock Memory efficient anonymous graph exploration.
\newblock In {\em {WG}}, volume 5344 of {\em Lecture Notes in Computer
  Science}, pages 14--29, 2008.

\bibitem{GollapudiP19}
S.~Gollapudi and D.~Panigrahi.
\newblock Online algorithms for rent-or-buy with expert advice.
\newblock In {\em {ICML}}, volume~97 of {\em Proceedings of Machine Learning
  Research}, pages 2319--2327. {PMLR}, 2019.

\bibitem{Gupta17}
R.~Gupta and T.~Roughgarden.
\newblock A {PAC} approach to application-specific algorithm selection.
\newblock {\em {SIAM} J. Comput.}, 46(3):992--1017, 2017.

\bibitem{HurkensW04}
C.~A.~J. Hurkens and G.~J. Woeginger.
\newblock On the nearest neighbor rule for the traveling salesman problem.
\newblock {\em Oper. Res. Lett.}, 32(1):1--4, 2004.

\bibitem{Im0QP21}
S.~Im, R.~Kumar, M.~M. Qaem, and M.~Purohit.
\newblock Non-clairvoyant scheduling with predictions.
\newblock In {\em {SPAA}}, pages 285--294. {ACM}, 2021.

\bibitem{KalyanasundaramP93}
B.~Kalyanasundaram and K.~Pruhs.
\newblock Constructing competitive tours from local information.
\newblock {\em Theor. Comput. Sci.}, 130(1):125--138, 1994.

\bibitem{Khalil17}
E.~B. Khalil, H.~Dai, Y.~Zhang, B.~Dilkina, and L.~Song.
\newblock Learning combinatorial optimization algorithms over graphs.
\newblock In {\em {NIPS}}, pages 6348--6358, 2017.

\bibitem{kolmogorov09}
V.~Kolmogorov.
\newblock Blossom {V:} a new implementation of a minimum cost perfect matching
  algorithm.
\newblock {\em Math. Program. Comput.}, 1(1):43--67, 2009.

\bibitem{KommKKS15}
D.~Komm, R.~Kr{\'{a}}lovic, R.~Kr{\'{a}}lovic, and J.~Smula.
\newblock Treasure hunt with advice.
\newblock In {\em {SIROCCO}}, volume 9439 of {\em Lecture Notes in Computer
  Science}, pages 328--341. Springer, 2015.

\bibitem{kool2018}
W.~Kool, H.~van Hoof, and M.~Welling.
\newblock Attention, learn to solve routing problems!
\newblock In {\em {ICLR} (Poster)}. OpenReview.net, 2019.

\bibitem{KumarPSSV19}
R.~Kumar, M.~Purohit, A.~Schild, Z.~Svitkina, and E.~Vee.
\newblock Semi-online bipartite matching.
\newblock In {\em {ITCS}}, volume 124 of {\em LIPIcs}, pages 50:1--50:20.
  Schloss Dagstuhl - Leibniz-Zentrum fuer Informatik, 2019.

\bibitem{LattanziLMV20}
S.~Lattanzi, T.~Lavastida, B.~Moseley, and S.~Vassilvitskii.
\newblock Online scheduling via learned weights.
\newblock In {\em {SODA}}, pages 1859--1877. {SIAM}, 2020.

\bibitem{LavastidaM0X21}
T.~Lavastida, B.~Moseley, R.~Ravi, and C.~Xu.
\newblock Learnable and instance-robust predictions for online matching, flows
  and load balancing.
\newblock In {\em {ESA}}, volume 204 of {\em LIPIcs}, pages 59:1--59:17.
  Schloss Dagstuhl - Leibniz-Zentrum f{\"{u}}r Informatik, 2021.

\bibitem{LawlerLRS1985}
E.~Lawler, J.~Lenstra, A.~{Rinnoy Kan}, and D.~Shmoys.
\newblock {\em The Traveling Salesman Problem -- A Guided Tour of Combinatorial
  Optimization}.
\newblock Wiley, 1985.

\bibitem{luperto2019}
M.~Luperto and F.~Amigoni.
\newblock Predicting the global structure of indoor environments: {A}
  constructive machine learning approach.
\newblock {\em Auton. Robots}, 43(4):813--835, 2019.

\bibitem{LykourisV18}
T.~Lykouris and S.~Vassilvitskii.
\newblock Competitive caching with machine learned advice.
\newblock In {\em {ICML}}, volume~80 of {\em Proceedings of Machine Learning
  Research}, pages 3302--3311. {PMLR}, 2018.

\bibitem{MahdianNS12}
M.~Mahdian, H.~Nazerzadeh, and A.~Saberi.
\newblock Online optimization with uncertain information.
\newblock {\em {ACM} Trans. Algorithms}, 8(1):2:1--2:29, 2012.

\bibitem{MedinaV17}
A.~M. Medina and S.~Vassilvitskii.
\newblock Revenue optimization with approximate bid predictions.
\newblock In {\em {NIPS}}, pages 1858--1866, 2017.

\bibitem{MegowMS12}
N.~Megow, K.~Mehlhorn, and P.~Schweitzer.
\newblock Online graph exploration: New results on old and new algorithms.
\newblock {\em Theor. Comput. Sci.}, 463:62--72, 2012.

\bibitem{Mitzenmacher20}
M.~Mitzenmacher.
\newblock Scheduling with predictions and the price of misprediction.
\newblock In {\em {ITCS}}, volume 151 of {\em LIPIcs}, pages 14:1--14:18.
  Schloss Dagstuhl - Leibniz-Zentrum f{\"{u}}r Informatik, 2020.

\bibitem{MiyazakiMO2009}
S.~Miyazaki, N.~Morimoto, and Y.~Okabe.
\newblock The online graph exploration problem on restricted graphs.
\newblock {\em {IEICE} Trans. Inf. Syst.}, 92-D(9):1620--1627, 2009.

\bibitem{osm}
{OpenStreetMap contributors}.
\newblock \url{ https://www.openstreetmap.org }, 2017.

\bibitem{PurohitSK18}
M.~Purohit, Z.~Svitkina, and R.~Kumar.
\newblock Improving online algorithms via {ML} predictions.
\newblock In {\em NeurIPS}, pages 9684--9693, 2018.

\bibitem{RaoIJW86}
N.~S.~V. Rao, S.~S. Iyengar, C.~C. Jorgensen, and C.~R. Weisbin.
\newblock Robot navigation in an unexplored terrain.
\newblock {\em J. Field Robotics}, 3(4):389--407, 1986.

\bibitem{Reinelt91}
G.~Reinelt.
\newblock {TSPLIB} - {A} traveling salesman problem library.
\newblock {\em {INFORMS} J. Comput.}, 3(4):376--384, 1991.

\bibitem{Rohatgi20}
D.~Rohatgi.
\newblock Near-optimal bounds for online caching with machine learned advice.
\newblock In {\em {SODA}}, pages 1834--1845. {SIAM}, 2020.

\bibitem{RosenkrantzSL13}
D.~J. Rosenkrantz, R.~E. Stearns, and P.~M. Lewis.
\newblock An analysis of several heuristics for the traveling salesman problem.
\newblock In {\em Fundamental Problems in Computing}, pages 45--69. Springer,
  2013.

\bibitem{Shalev2014}
S.~Shalev{-}Shwartz and S.~Ben{-}David.
\newblock {\em Understanding Machine Learning - From Theory to Algorithms}.
\newblock Cambridge University Press, 2014.

\bibitem{tsplib}
TSPLib.
\newblock \url{http://comopt.ifi.uni-heidelberg.de/software/TSPLIB95/tsp/},
  2021.

\bibitem{Valiant84}
L.~G. Valiant.
\newblock A theory of the learnable.
\newblock {\em Commun. {ACM}}, 27(11):1134--1142, 1984.

\bibitem{Vapnik1992}
V.~Vapnik.
\newblock Principles of risk minimization for learning theory.
\newblock In {\em {NIPS}}, pages 831--838. Morgan Kaufmann, 1991.

\bibitem{vapnik1999}
V.~Vapnik.
\newblock An overview of statistical learning theory.
\newblock {\em {IEEE} Trans. Neural Networks}, 10(5):988--999, 1999.

\bibitem{vapnik1971}
V.~N. Vapnik and A.~Y. Chervonenkis.
\newblock On the uniform convergence of the frequencies of occurrence of events
  to their probabilities.
\newblock In {\em Empirical Inference}, pages 7--12. Springer, 2013.

\bibitem{Vinyals15}
O.~Vinyals, M.~Fortunato, and N.~Jaitly.
\newblock Pointer networks.
\newblock In {\em {NIPS}}, pages 2692--2700, 2015.

\bibitem{zhou2020}
J.~Zhou, G.~Cui, S.~Hu, Z.~Zhang, C.~Yang, Z.~Liu, L.~Wang, C.~Li, and M.~Sun.
\newblock Graph neural networks: {A} review of methods and applications.
\newblock {\em {AI} Open}, 1:57--81, 2020.

\end{thebibliography}
	
\end{document}